\documentclass[letterpaper, 10 pt, conference]{ieeeconf}
\IEEEoverridecommandlockouts  
\usepackage{amsmath}
         
\usepackage{amsthm}
\usepackage{amssymb}  
\usepackage{mathtools}
\usepackage{bm}
\usepackage{makecell}
\let\labelindent\relax
\usepackage{enumitem}

\usepackage{url}
\usepackage{import}
\usepackage{pgf}
\usepackage{pgfplots}
\pgfplotsset{compat=1.18}   
\usepackage{color}
\usepackage{tabularx} 
\usepackage{multirow}
\usepackage{cite}
\usepackage{algorithm} 
\usepackage{algpseudocode} 
\usepackage{empheq}
\usepackage{multirow}
\usepackage{float}
\usepackage{setspace}
\usepackage{placeins}
\usepackage[caption=false,font=footnotesize]{subfig}
\usepackage{tikz}
\newcommand{\scalebarimg}[4]{%
    \begin{tikzpicture}[every node/.style={inner sep=0,outer sep=0}]%
        \draw node[name=micrograph] {\includegraphics[width=#2]{#1}};
        \draw (micrograph.north west) node[anchor=north west,yshift=-1,#4]{\small{#3}};
    \end{tikzpicture}%
}

\usepackage{balance} 

\makeatletter
\let\NAT@parse\undefined
\makeatother
\usepackage{hyperref} 
\hypersetup{
    colorlinks=true,
    citecolor=black,
    linkcolor=black,
    urlcolor=black,
    bookmarks=false,
    hypertexnames=true
}

\newlength{\mycaption}
\newlength{\mySepBetweenFigAndCap}
\newlength{\mySepBetweenTopAndFig}
\newtheorem{theorem}{Theorem}

\theoremstyle{definition}

\title{\LARGE \bf
Enhanced Robotic Navigation in Deformable Environments using Learning from Demonstration and Dynamic Modulation
}
\author{%
    Lingyun Chen$^{1*}$, 
    Xinrui Zhao$^{1*}$, 
    Marcos P. S. Campanha$^{1}$,
    Alexander Wegener$^{1}$,
    Abdeldjallil Naceri$^{1}$,\\ 
    Abdalla Swikir$^{2}$ and 
    Sami Haddadin$^{2}$
\thanks{This work was supported by the German Research Foundation (DFG) as part of Germany’s Excellence Strategy, EXC 2050/1, Project ID 390696704 – Cluster of Excellence ``Centre for Tactile Internet with Human-in-the-Loop'' (CeTI) of Technische Universität Dresden.}%
\thanks{$^{1}$Munich Institute of Robotics and Machine Intelligence (MIRMI), Technical University of Munich (TUM), Germany. Corresponding author:
{\tt\small lingyun.chen@tum.de}}%
\thanks{$^{2}$Mohamed Bin Zayed University of Artificial Intelligence, Abu Dhabi, UAE}
\thanks{$^{*}$The first two authors contributed equally to this work.}
}

\begin{document}
	
\maketitle
\thispagestyle{plain}
\pagestyle{plain}

\setstretch{0.93}

\begin{abstract}

This paper presents a novel approach for robot navigation in environments containing deformable obstacles. 
By integrating Learning from Demonstration (LfD) with Dynamical Systems (DS), we enable adaptive and efficient navigation in complex environments where obstacles consist of both soft and hard regions. 
We introduce a dynamic modulation matrix within the DS framework, allowing the system to distinguish between traversable soft regions and impassable hard areas in real-time, ensuring safe and flexible trajectory planning.
We validate our method through extensive simulations and robot experiments, demonstrating its ability to navigate deformable environments. Additionally, the approach provides control over both trajectory and velocity when interacting with deformable objects, including at intersections, while maintaining adherence to the original DS trajectory and dynamically adapting to obstacles for smooth and reliable navigation. 
%


\end{abstract}

\section{Introduction}

Navigating complex environments remains a key challenge in robotics, particularly in scenarios requiring enhanced decision-making and adaptability. While most current research emphasizes obstacle avoidance by treating all obstacles as rigid entities to be avoided entirely \cite{kunchev2006path, fox1997dynamic, lamiraux2004motion}, relatively little attention has been given to environments containing deformable regions that could be incorporated into the robot's path planning. This gap in existing approaches leads to inefficient navigation, especially in industrial, domestic, and healthcare applications where the interaction with soft objects is crucial for effective operation \cite{Li_2024, Zhu2023SafeContact}.

\begin{figure}[!t]%
    \vspace*{0.2em}%
    \centering%
    \includegraphics[width=0.73\columnwidth]{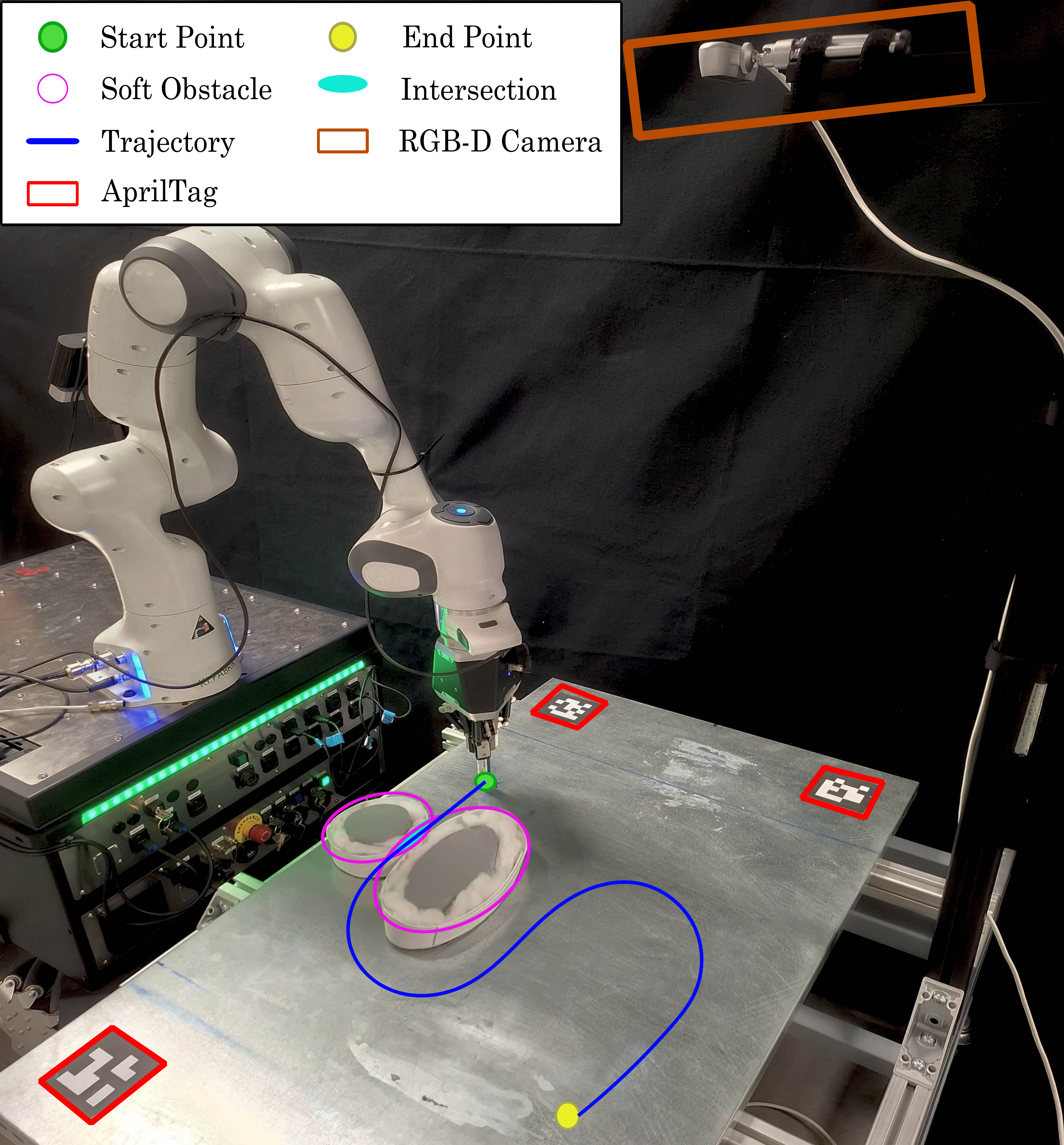}%
    \vspace*{\mySepBetweenFigAndCap}%
    \caption{
    Experimental setup showing the robot's navigation from the start point (green circle) along the planned trajectory (blue line) toward the attractor (yellow circle). The robot interacts at the intersection (cyan ellipse) of two deformable obstacles (pink ellipses). An RGB-D camera (orange box) provides continuous perception of obstacle poses and properties.
    }
    \label{fig:physical_obs1}%
    \vspace*{\mycaption + \mySepBetweenFigAndCap}%
    \vspace{-0.65em}
\end{figure}

Learning from Demonstration (LfD) enables robots to learn tasks through demonstrations, allowing flexible adaptation without explicit programming \cite{atkeson1997robot, Calinon2018, zhu2018robot}. This approach is especially valuable for tasks challenging to define through conventional programming or where demonstrating the desired outcome is more efficient. In LfD, trajectories are commonly encoded using Dynamical Systems (DS) to ensure stability and consistency throughout the task execution \cite{khansari2010imitation}. Position and velocity data from demonstrations are optimized to create a stable DS that reliably reaches the target, regardless of initial conditions or disturbances \cite{figueroa2018APhysically-ConsistentBayesian, khansari2011LearningStableNon-Linear}. However, current methods focus on rigid object avoidance, overlooking deformable ones \cite{khansari2012Adynamicalsystemapproach, huber2019avoidance}. 


Traditional obstacle avoidance methods, such as Artificial Potential Fields (APF) \cite{khatib1986realtime, barraquand1992numerical}, are limited by static models and the problem of local minima. While techniques like Harmonic Potential Functions (HPF) avoid these issues by modeling fluid dynamics \cite{kim1992real, feder1997real}, they still lack the capability to differentiate between deformable and rigid obstacles. Moreover, LfD and DS offer promising frameworks for robotic navigation, but they require adaptations to effectively handle environments containing such obstacles.

To address these limitations and build upon the existing dynamic modulation method \cite{khansari2012Adynamicalsystemapproach}, we present a novel method that integrates LfD with DS using an enhanced dynamic modulation matrix to continuously adjust the robot’s trajectory based on the stiffness properties of obstacles. 



In summary, the contributions of this paper are as follows:

\begin{enumerate}
  \item We integrate LfD with DS, leveraging real-time perception to estimate the deformable obstacles' pose and properties. An enhanced dynamic modulation matrix is employed to distinguish between soft and hard regions, enabling the robot to dynamically adapt its trajectory by traversing soft regions while avoiding hard ones.
  \item We demonstrate the effectiveness of our approach through extensive simulations and robot experiments illustrated in Fig. \ref{fig:physical_obs1}, showing capabilities of navigating in deformable environments and benefits in providing control over trajectory and velocity when interacting with deformable objects, including at intersections, compared to existing obstacle avoidance algorithms. The source code for this research is
publicly accessible. \footnote{%
Available at \href{https://gitlab.com/ds_experiment/ds_obstacle}
{https://gitlab.com/ds\_experiment/ds\_obstacle}}
\end{enumerate}

\section{Preliminaries}\label{sec:preliminaries}
\subsection{Stable Dynamical System}

A stable DS with an equilibrium point in  $\bm{\xi}^\star \in \mathbb{R}^d$ can be expressed as an optimization problem with the dynamics function $f :  \mathbb{R}^d \to  \mathbb{R}^d$  and the Lyapunov function $V :  \mathbb{R}^d \to  \mathbb{R}$ as the variables of the problem. Given reference data in the form of N tuples of positions and velocity, $(\mathrm{\bm{\xi}}_{ref}^{(i)} , \dot{\mathrm{\bm{\xi}}}_{ref}^{(i)}) , i = 1,...,\mathit{N}$, one needs to solve the following the optimization problem (cf.\cite{khansari2011LearningStableNon-Linear}):

\begin{align}
&\min_{f}\sum_{i=1}^{N} \left\|\dot{\mathrm{\bm{\xi}}}_{\mathrm{ref}}^{(i)} - f(\mathrm{\bm{\xi}}_{\mathrm{ref}}^{(i)})\right\|  \\
&\text{s.t. } f(\bm{\xi}^\star)=0   \notag\\
&V(\bm{\xi}^\star)=0, V(\bm{\xi})>0, \text{ for all } \bm{\xi} \in \mathbb{R}^d \setminus \{\bm{\xi}^\star\}  \notag\\
&\dot{V}(\bm{\xi}^\star)=0, \dot{V}(\bm{\xi})<0, \text{ for all } \bm{\xi} \in \mathbb{R}^d \setminus \{\bm{\xi}^\star\}, \notag
\end{align}
where $V(\bm{\xi})$ is the corresponding Lyapunov function \cite{khalil2002nonlinear}. The approach using a Linear Parameter Varying (LPV) framework, combined with a Gaussian Mixture Model (GMM)-based mixing function, provides a formulation for the nonlinear DS $f$:
\begin{align}
\dot{\bm{\xi}}=f(\bm{\xi})=\sum_{k=1}^{K}\gamma_k(\bm{\xi})(\bm{A}_k\bm{\xi}+\bm{b}_k),
\label{eq:DS_formel}
\end{align}
learned from demonstrations in a decoupled manner, where $\bm{A}_k \in \mathbb{R}^{d\times d}$ and $\bm{b}_k \in \mathbb{R}^{d}$. The GMM parameters $\gamma_k : \mathbb{R}^d \to \mathbb{R}$ are defined as \cite{figueroa2018APhysically-ConsistentBayesian} :
\begin{align}
\gamma_k(\bm{\xi}) = \frac{\pi_k \, p(\bm{\xi}|k)}{\sum_j \pi_j \, p(\bm{\xi}|j)},
\end{align}
with ${\pi_k\ge 0}$ being the mixing weights satisfying ${\sum_{k=1}^{K} \pi_k = 1}$, and %
$p(\bm{\xi}|k), k = 1,...,K,$ being the probability density function of a multivariate normal distribution.
The linear parameters $\theta_\gamma = \left\{ (\bm{A}_k, \bm{b}_k) \right\}_{k=1}^{K}$ are estimated via semidefinite optimization problem that ensures global asymptotic stability of the system. Therefore, the linear parameters are derived from Parametrized Quadratic Lyapunov Function (P-QLF):
\begin{align}
 V(\bm{\xi}) &= (\bm{\xi}-\bm{\xi}^\star)^{T}\bm{P}(\bm{\xi}-\bm{\xi}^\star),\\
 \bm{P} &= \bm{P}^{T} > 0.
 \end{align}
 where $\bm{P}$ signifies a positive-definite matrix, ensuring that $V(\bm{\xi})$ exhibits a strictly convex nature, which is a basis for confirming the system's stability.
 

\subsection{Obstacle Avoidance Mechanism}

Following the method in \cite{khansari2012Adynamicalsystemapproach}, a continuous distance function is introduced as $\Gamma(\bm{\xi}) : \mathbb{R}^d \setminus X^r \rightarrow \mathbb{R}$ for each obstacle, where \(X^r \subset \mathbb{R}^d\) represents a region. This function is important for delineating three distinct regions near an obstacle, thereby facilitating the classification of points based on their spatial relation to it. The regions are defined as follows:
\begin{align}
&\text{Exterior Region: } &X^e &= \left\{\bm{\xi} \in \mathbb{R}^d | \Gamma(\bm{\xi}) > 1\right\}, \notag\\
&\text{Boundary Region: } &X^b &= \left\{\bm{\xi} \in \mathbb{R}^d | \Gamma(\bm{\xi}) = 1\right\}, \notag\\
&\text{Interior Region: } &X^i &= \left\{\bm{\xi} \in \mathbb{R}^d \setminus (X^e \cup X^b)\right\}.
\end{align}
The function $\Gamma(\cdot)$ is specifically engineered to increase monotonically as the distance from the obstacle's center $\bm{\xi}^c$ expands, further characterized by a continuous first-order partial derivative ($C^1$ smoothness). In this work, $\Gamma(\bm{\xi})$ is defined as:
\begin{align}
\Gamma(\bm{\xi}) = \sum_{i=1}^{d} \left(\frac{\left\|\bm{\xi}_i - \bm{\xi}_i^{c}\right\|}{  {R}(\bm{\xi})}\right)^{2p},
\quad  p \in \mathbb{N}^+,
\label{eq:Gamma_original}
\end{align}    
where $R(\bm{\xi})$ represents the distance from a center point $\bm{\xi}^c$ within the obstacle to the boundary defined by $\Gamma(\bm{\xi})=1$, aligned in the direction $r(\bm{\xi})$. This formulation indicates the obstacle avoidance strategy by enabling a clear distinction of navigable space concerning the encountered obstacles.

\section{Soft Obstacle Avoidance Algorithm}
Real-time obstacle avoidance is facilitated by applying a dynamic modulation matrix to the predefined DS, as delineated in Eq~\eqref{eq:DS_formel}. The modified system dynamics are represented by\cite{khansari2012Adynamicalsystemapproach}:
\begin{align}
\dot{\bm{\mathrm{\bm{\xi}}}} = \bm{M}(\mathrm{\bm{\xi}})f(\mathrm{\bm{\xi}}) \!\quad\! \text{with} \!\quad\! \bm{M}(\bm{\mathrm{\bm{\xi}}}) = \bm{E}(\mathrm{\bm{\xi}})\bm{D}(\bm{\xi})\bm{E}(\bm{\xi})^{-1}.
\label{eq:avoidance_main}
\end{align}
Here, $\bm{M}(\mathrm{\bm{\xi}})$ is the local modulation matrix, which offers substantial benefits by keeping existing extrema, such as attractors, and excluding the emergence of new extrema,  assuming that \(\bm{M}(\bm{\xi})\) maintains full rank. The utilization of the modulation matrix provides a closed-form solution, ensuring that the matrix's application to the system remains computationally efficient.
\begin{figure}[h]
    \vspace*{\mySepBetweenTopAndFig}%
    \vspace{-3mm}
    \centering
    \subfloat[]{\label{fig:obs.1}\scalebarimg{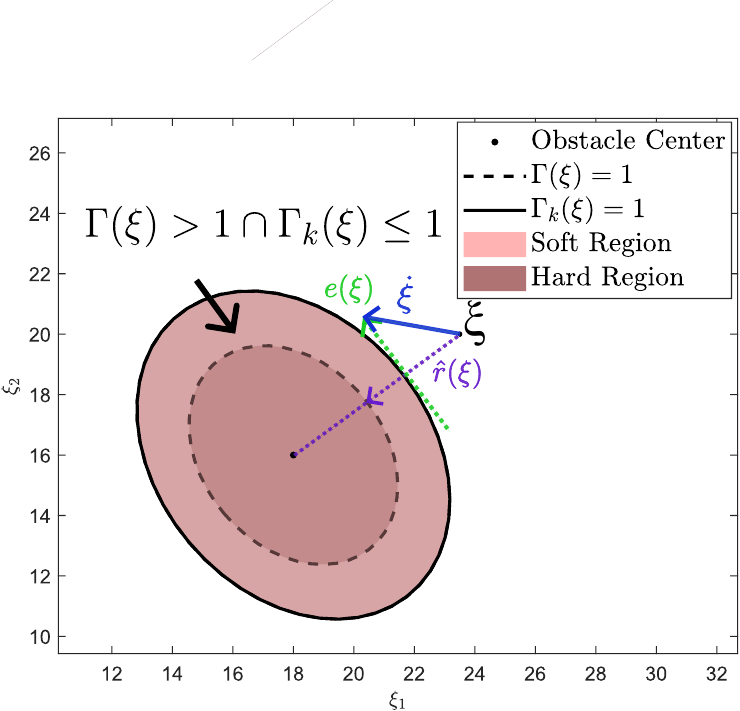}{0.48\linewidth}{}{black}}%
    \hspace{\fill}%
    \subfloat[]{\label{fig:obs.2}\scalebarimg{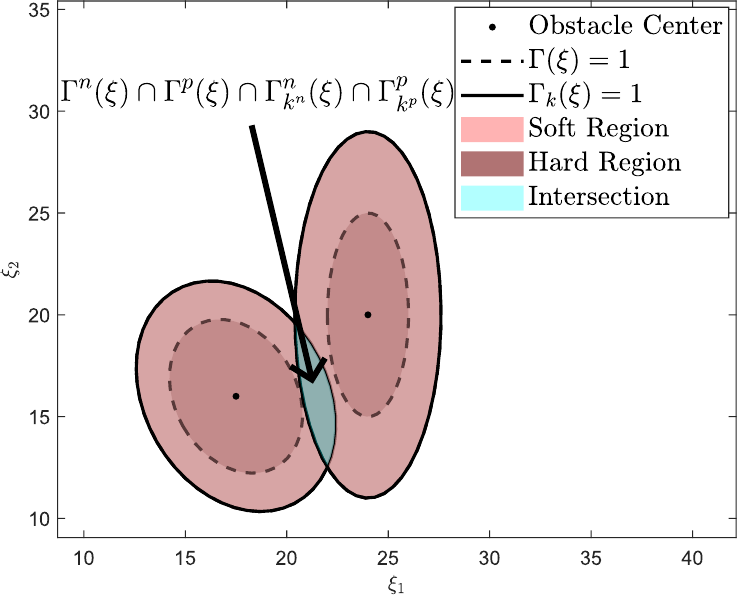}{0.48\linewidth}{}{black}}%
    \hspace{\fill}%
    
    \caption{(a) The deformable obstacle is defined as comprising soft and hard regions., represented by the boundaries \(\Gamma(\bm{\xi})=1\) and \(\Gamma_{k}(\bm{\xi})=1\), and by the regions constrained by \( \Gamma(\bm{\xi}) > 1 \cap \Gamma_{k}(\bm{\xi}) \leq 1 \). Stretching along the tangent vectors ${\bm{e}}(\bm{\xi})$ and compressing in the reference direction $\hat{\bm{r}}(\bm{\xi})$ allow the avoidance of hard regions with \(\dot{\bm{\xi}}\). (b) For intersecting obstacles, overlapping soft regions are represented by the distance functions.}
    \label{fig:obs_single_intersecting}
    \vspace*{\mySepBetweenFigAndCap}%
\end{figure}

The modulation matrix \(\bm{M}(\bm{\xi})\) is composed of a basis matrix:
\begin{align}
\!\!\!\bm{E}(\bm{\xi}) \!=\! \left[\hat{\bm{r}}(\bm{\xi}) \quad \bm{e}_1 (\bm{\xi}) \!\!\!\quad\!\! \ldots \!\!\quad\!\! \bm{e}_{d-1} (\bm{\xi})\right],
\hat{\bm{r}}(\bm{\xi}) \!=\! \frac{\bm{\xi} - \bm{\xi}^r}{\|\bm{\xi} - \bm{\xi}^r\|} ,\!\!
\label{eq:avoidance_matrixE}
\end{align} 
where $\hat{\bm{r}}(\bm{\xi})$ is the unit vector in the reference direction, and the tangents $\bm{e}_{(\cdot)}(\bm{\xi})$ form a $(d - 1)$-dimensional orthonormal basis to the gradient of the distance function $\frac{d\Gamma(\bm{\xi})}{d\bm{\xi}}$ (Fig. \ref{fig:obs.1}). The reference direction $\bm{r}(\bm{\xi})$ is based on a reference point inside the obstacle. 

The associated eigenvalue matrix $\bm{D}(\bm{\xi})$ stretches and compresses the dynamics along the directions defined by the $\bm{r}-\bm{e}-$ system. This diagonal matrix enables individual modification of the reference $\bm{r}(\bm{\xi})$ and tangent $\bm{e}_{(\cdot)}(\bm{\xi})$ directions:
\begin{align}
\bm{D}(\bm{\xi}) = \text{diag} (\lambda_r (\bm{\xi}), \lambda_{e_1} (\bm{\xi}), \ldots, \lambda_{e_{d-1}} (\bm{\xi}))
\label{eq:avoidance_matrixD},
\end{align} 
where the eigenvalues $\lambda_{(\cdot)}(\bm{\xi})$ determine the degree of stretching in each respective direction.

\subsection{Eigenvalues}
The eigenvalue associated to the ﬁrst eigenvector decreases to zero on the obstacle’s hull, cancelling the ﬂow in the direction of the obstacle. In obstacle avoidance systems, modifying the flow of movement near obstacles is crucial for ensuring safe navigation. By adjusting the eigenvalues associated with different directions, these systems can control the robot's motion, preventing collisions while maintaining efficient movement around obstacles \cite{khansari2012Adynamicalsystemapproach}.
We have:
\begin{align}
    0 \leq \lambda_r(\bm{\xi}) \leq 1, \quad \lambda_e(\bm{\xi}) \geq 1, \nonumber\\
    \lambda_r(\bm{\xi} \mid \Gamma(\bm{\xi}) = 1) = 0, \nonumber\\  \underset{\Gamma(\bm{\xi})}{\text{argmax}}\, \lambda_e(\bm{\xi}) = 1,  \lim_{\Gamma(\bm{\xi}) \rightarrow \infty} \lambda_{(\cdot)}(\bm{\xi}) = 1 
    \label{eq:avoidance_eigenvalues1}
\end{align}
Using the distance function $\Gamma(\bm{\xi})> 0$, which measures the distance to the obstacle's surface:

\begin{align}
    \!\!\lambda_r(\bm{\xi}) := 1 - \frac{1}{\Gamma(\bm{\xi})}, \!\!\quad\!\! \lambda_e(\bm{\xi}) := 1 + \frac{1}{\Gamma(\bm{\xi})}.
    \label{eq:avoidance_eigenvalues2}
\end{align}

\subsection{Stiffness Coefficient}

A deformable obstacle comprises hard and soft regions (as shown in Fig. \ref{fig:obs.1}), with a stiffness coefficient \( k \). This coefficient depends on the geometric properties of the obstacle, specifically the axis lengths \( a, b \in \mathbb{R} \) corresponding to the hard and soft regions, respectively. It is defined by the following expression:

\begin{align}
k &= \exp\left(\frac{b}{a} - 1\right), \quad \text{where} \quad \frac{b}{a} > 1, \quad a \neq 0.
\label{eq:stiffness_coe}
\end{align}

The stiffness coefficient \( k \) quantitatively characterizes the rigidity or compliance of obstacles, where higher values of 
\( k \) represents more compliant (softer) regions, and lower values indicate more rigid (harder) regions. This distinction is essential for enabling the robot to adapt its trajectory according to the varying stiffness properties of the environment. The assumption $\frac{b}{a} > 1$ is physically justified as it ensures that the stiffness coefficient 
\( k \)  remains within a range that preserves system stability, thereby preventing numerical instability and unrealistic behavior.

To further describe the influence of the deformable obstacles on the robot's motion, we introduce the function $\Gamma_k(\bm{\xi}) : \mathbb{R}^d \setminus X^r \rightarrow \mathbb{R}$, which is parameterized by the stiffness coefficient \( k \) and determines the obstacle's impact on the trajectory modification:

\begin{align}
\!\!\Gamma_{k}(\bm{\xi}) \!= \!\sum_{i=1}^{d} \left(\frac{\left\|\bm{\xi}_i - \bm{\xi}_i^{c}\right\|}{(\ln(k) + 1) \, R(\bm{\xi})}\right)^{2p}, \!\!\!\quad i = 1, 2, \ldots, d,
\label{eq:Gamma_stiffness}
\end{align}
where \( p \in \mathbb{N}^+ \) is a positive integer, \( \bm{\xi}_i \) represents the position vector component along the \( i \)-th dimension, \( \bm{\xi}_i^{c} \) denotes the center of the obstacle along that dimension.

This formulation allows the robot to adaptively adjust its trajectory by accounting for the varying resistance imposed by different obstacle regions. Regions with higher stiffness coefficients (softer regions) exert less influence on the trajectory, enabling more efficient navigation in complex environments with obstacles possessing various stiffness properties.

\subsection{Adaptive Motion Strategy}

Designing trajectories that prioritize navigation through softer regions of  obstacles is beneficial, as it minimizes the overall cost of navigation in such environments. Soft regions are defined where \( \Gamma^n(\bm{\xi}) > 1 \) and \( 0 < \Gamma_{k^n}^n(\bm{\xi}) \leq 1 \). Here, \( n \) represents the number of objects, with \( N \) obstacles in total. The proposed approach dynamically adapts the robot's trajectory by considering interactions with \( N \) obstacles while ensuring convergence to the desired target.
The updated expression for the velocity vector \( \dot{\bm{\xi}}' \), which accounts for these interactions, is given by:

\begin{align}
\dot{\bm{\xi}}' = \dot{\bm{\xi}} + \sum_{n=1}^{N} \left( \operatorname{S}(\theta_2, \theta_1^n) \, \bm{R}(\theta_2) \, \hat{\bm{r}}^n(\bm{\xi}) \, c \left( \frac{1}{\Gamma_{k^n}^n(\bm{\xi})} \right)^2 \right),
\label{eq:speedup_velocity}
\end{align}
where \( \operatorname{S}(\theta_2, \theta_1^n) \) is a sign function defined as:

\[
\operatorname{S}(\theta_2, \theta_1^n) = \operatorname{sgn}(\cos(\theta_2 - \theta_1^n)),
\]


In this formulation, $\theta_1^n$ denotes the rotation angle of the obstacle’s local frame, while $\theta_2$ is the desired rotation angle representing the deviation of the robot’s trajectory from the reference direction \( \hat{\bm{r}}^n(\bm{\xi}) \). This reference direction is defined as the unit vector, relative to a center point of the obstacle, and serves to guide the trajectory while avoiding unnecessary detours through soft regions. The rotation matrix $\bm{R}(\theta_2)$ is used to align the modulation with the desired trajectory direction. The modulation factor \( c \) can modify the velocity based on the robot's motion limitation.

\begin{theorem}
Consider a \( d \)-dimensional static deformable hyper-sphere obstacle in \( \mathbb{R}^d \) with center \( \boldsymbol{\bm{\xi}}^c \), stiffness coefficient \( k \), and radius \( r_o \). The deformable obstacle boundary consists of the hyper-surface \( X^{br} \subset \mathbb{R}^d \) defined as:
\[
X^{br} = \left\{ \boldsymbol{\bm{\xi}} \in \mathbb{R}^d : \left\|
{\boldsymbol{\bm{\xi}}} - \boldsymbol{\bm{\xi}}^c\right\| = (\ln(k) + 1) \, r_o \right\}.
\]    
Any motion \( \left\{ \boldsymbol{\bm{\xi}}_t \right\}_{t=0}^\infty \) that starts outside the obstacle's hard region and evolves according to Eq.~\eqref{eq:avoidance_main} enters the obstacle's soft region.
\end{theorem}

\begin{proof}
Continuity is maintained after entering the obstacle's soft region by controlling the motion at the boundary. Specifically, the normal velocity at the boundary points ensures a smooth transition during the interaction. \(\bm{\xi}^{br} \in X^{br}\) does not vanish:
\begin{align}
\bm{n}(\mathrm{\bm{\xi}}^{br})^T \dot{\mathrm{\bm{\xi}}}^{br} &\neq 0 \quad \forall \mathrm{\bm{\xi}}^{br} \in X^{br}, 
\label{pr:notequal_0}
\end{align}
where the normal vector is defined as:
\begin{align}
\bm{n}(\mathrm{\bm{\xi}}^{br}) &= \frac{\mathrm{\bm{\xi}}^{br} - \mathrm{\bm{\xi}}^c}{\|\mathrm{\bm{\xi}}^{br} - \mathrm{\bm{\xi}}^c\|}, \quad \tilde{\mathrm{\bm{\xi}}}^{br} = \mathrm{\bm{\xi}}^{br} - \mathrm{\bm{\xi}}^c, \notag \\
\Rightarrow \bm{n}(\tilde{\mathrm{\bm{\xi}}}^{br}) &= \frac{\tilde{\mathrm{\bm{\xi}}}^{br}}{(\ln(k) + 1)\, r_o} \quad \forall \mathrm{\bm{\xi}}^{br} \in X^{br}.
\label{pr:normal_vector}
\end{align} 

Substituting Eq.~\eqref{eq:speedup_velocity} into Eq.~\eqref{pr:notequal_0}, where \(\mathrm{\bm{\xi}}^{b}\) is the original velocity vector based on the dynamic modulation matrix,

\begin{align}
\mathbf{n}(\tilde{\bm{\xi}}^{br})^{T} \Big(& \dot{\bm{\xi}}^{b} \pm \mathbf{R}(\theta_2)\,\mathbf{n}(\tilde{\bm{\xi}}^{br})\,\hat{\bm{r}}\Big(\frac{\bm{\xi}^{b}}{\ln(k) + 1}\Big) \nonumber\\[1ex]
& \quad\cdot  c \left(\frac{1}{\Gamma_{k^n}^n(\bm{\xi})}\right)^2 \Big) \neq 0.
\label{summarized}
\end{align}

There are two distinct parts, which are evident from Eq.~\eqref{summarized} that the expression is non-zero:
\begin{align}
\bm{n}(\tilde{\mathrm{\bm{\xi}}}^{br})^T \dot{\mathrm{\bm{\xi}}}^{b} &\neq 0, 
\label{pr:ndotvec_notequal0}
\end{align} 
\begin{align}
\pm \bm{n}(\tilde{\mathrm{\bm{\xi}}}^{br})^T \left( \bm{R}(\theta_2) \, \hat{\bm{r}}\left(\frac{\mathrm{\bm{\xi}}^{b}}{(\ln(k) + 1)}\right) \, c \left( \frac{1}{\Gamma_{k}(\bm{\xi}^{b})} \right)^2 \right) &\neq 0.
\label{pr:extramodified}
\end{align}

Substituting Eqs.~\eqref{pr:normal_vector} and~\eqref{eq:avoidance_main} into Eq.~\eqref{pr:ndotvec_notequal0} yields:
\begin{equation}
\bm{n}(\mathrm{\bm{\xi}}^{br})^T \dot{\mathrm{\bm{\xi}}}^b
= \frac{\left(\mathrm{\bm{\xi}}^b\right)^T}{(\ln(k) + 1)\, r_o} \bm{E}\left(\mathrm{\bm{\xi}}^b\right) 
\bm{D}\left(\mathrm{\bm{\xi}}^b\right) \bm{E}\left(\mathrm{\bm{\xi}}^b\right)^{-1} f(\cdot).
\label{pr:ndotEDE}
\end{equation}  

Given that \(\bm{n}(\tilde{\mathrm{\bm{\xi}}}^{br})\) corresponds to the first eigenvector of \(E\left(\mathrm{\bm{\xi}}^b\right)\), Eq.~\eqref{pr:ndotEDE} simplifies to:
\begin{equation}
\bm{n}(\mathrm{\bm{\xi}}^{br})^T \dot{\mathrm{\bm{\xi}}}^b = 
\begin{bmatrix} (\ln(k) + 1)\, r_o \\ 0_{d-1} \end{bmatrix}^T \bm{D}\left(\mathrm{\bm{\xi}}^b\right) \bm{E}\left(\mathrm{\bm{\xi}}^b\right)^{-1} f(\cdot)
\label{pr:matrixdotDE}
\end{equation}

For all points on the boundary of soft region, the first eigenvalue is non-zero, deviating from the original boundary, i.e., \(\lambda^1 \neq 0\), \(\forall \mathrm{\bm{\xi}}^{br} \in X^{br}\). Consequently, \(\bm{n}(\tilde{\mathrm{\bm{\xi}}}^{br})^T \dot{\mathrm{\bm{\xi}}}^{br} \neq 0\). 
\hfill 
\end{proof}



%
%

\subsection{Intersection}

When multiple deformable obstacles intersect within the interaction space, the combined regions of influence create complex dynamics that affect the robot's movement. Specifically, for each pair of intersecting obstacles (as shown in Fig. \ref{fig:obs.2}) \( n \) and \( p \), where \( p \neq n \) and \( p = 1, \ldots, N \), consider the region where following these distance limits.

\[
\begin{cases}
\Gamma^n(\bm{\xi}) > 1, \\
\Gamma^p(\bm{\xi}) > 1, \\
0 < \Gamma_{k^n}^n(\bm{\xi}) \leq 1, \\
0 < \Gamma_{k^p}^p(\bm{\xi}) \leq 1.
\end{cases}
\]

When the robot passes through the intersection of two obstacles, the external forces increase due to soft obstacle deformation, potentially causing the robot to become stuck. To ensure safety and smoothness during passage through the intersecting region, we reduce the robot’s velocity to allow an additional safety margin and reduce the required actuation force. 

The updated expression for the velocity vector \( \dot{\bm{\xi}}' \), which incorporates the effects of intersecting obstacles, is given by:
\begin{align}
\dot{\bm{\xi}}' = \dot{\bm{\xi}} 
& -  \sum_{n=1}^{N} \sum_{\substack{p=1 \\ p \neq n}}^{N} \left( \operatorname{S}(\theta_2, \theta_1^n) \, \bm{R}(\theta_2) \, \hat{\bm{r}}^n(\bm{\xi}) \cdot\bm{e}_1^n (\bm{\xi}) \right) \notag \\
& \qquad \cdot \left( c \left( \frac{2}{\Gamma_{k^n}^n(\bm{\xi}) + \Gamma_{k^p}^p(\bm{\xi})} \right)^2 \right),
\label{eq:slowdown_velocity}
\end{align}

In this formulation, \( \hat{\bm{r}}^n(\bm{\xi}) \) represents the unit vector in the reference direction associated with obstacle \( n \). The $\bm{e}_1^n (\bm{\xi})$ represents a $(d - 1)$-dimensional orthonormal set that lies in a space perpendicular to the gradient of the distance function. The function \( \Gamma_{k^n}^n(\bm{\xi}) \) and \( \Gamma_{k^p}^p(\bm{\xi}) \) represent the obstacle-specific impact parameters based on their stiffness and spatial configuration. The adjustment is made by summing contributions from all obstacle pairs, where each contribution is modulated by a factor that depends on the relative strengths or distances of the obstacles. 

The presence of intersecting obstacles introduces additional resistance due to the overlap of their influence regions. Specifically, when multiple deformable obstacles intersect, their combined effects increase external resistance, reducing the robot's movement. This approach effectively manages the complexity of interactions among various obstacles, leading to more realistic and controlled navigation through the interaction space. 

\subsection{Safety Margin}
Taking a more comprehensive approach is crucial when encountering a higher density of obstacles. For example, prioritizing the avoidance of more rigid obstacles is essential to ensure safety in real-world scenarios. To address this, we incorporate a safety factor with predefined limits to ensure successful navigation while maintaining safety. The desired safety margin around an object can be obtained by scaling the state variable (in the obstacle frame of reference) in the dynamic modulation matrix $\bm{M}({\bm{\xi}})$ given as follows:
\begin{equation}
    \bm{M}({\bm{\xi}}_\eta) = \bm{E}({\bm{\xi}}_\eta) \bm{D}({\bm{\xi}}_\eta) \bm{E}({\bm{\xi}}_\eta)^{-1}
    \label{eq:safety}
\end{equation}
where \( {\bm{\xi}}_\eta = \frac{{\bm{\xi}}}{\eta} \) represents the element-wise division of \( \bm{\xi} \) by \( \eta \in \mathbb{R}^d \), and \( \eta_i \geq 1 \) for all \( i \in \{1, \ldots, d\} \) is the desired safety factor. This safety factor inflates the object along each direction \( \tilde{\bm{\xi}}_i \) by the magnitude \( \eta_i \) (in the obstacle frame of reference). By selecting different values for each \( \eta_i \), one can control the required safety margin along the corresponding direction of the object, thus ensuring safer and more controlled navigation through dense obstacle environments.

\subsection{Moving Obstacles}


The modulation for moving obstacles described is based on the approach proposed in \cite{huber2019avoidance}. The modulation is performed in the obstacle reference frame, and then transformed to the inertial frame \cite{khansari2012dynamical}.
For deformable obstacles that can move but maintain an overall fixed shape, we keep the same stiffness coefficient $k$ to ensure consistency in the robot's interaction regardless of its motion. By performing the modulation within the obstacle’s reference frame, 
whose origin is defined by the obstacle's center \(\bm{\xi}^c\), we obtain the relative position $
\tilde{\bm{\xi}} = \bm{\xi} - \bm{\xi}^c.$
The robot's relative velocity with respect to the obstacle in the obstacle reference frame, \(\dot{\tilde{\bm{\xi}}}\), is derived from its linear and angular velocities in the obstacle reference frame, \(\dot{\bm{\xi}}^{L,o}\) and \(\dot{\bm{\xi}}^{R,o}\), via
\begin{align}
\dot{\tilde{\bm{\xi}}} = \dot{\bm{\xi}}^{L,o} + \dot{\bm{\xi}}^{R,o} \times \tilde{\bm{\xi}}.
\end{align}
The robot’s desired velocity \(\dot{\bm{\xi}}\) is obtained by subtracting the obstacle’s motion within the DS, as follows:

\begin{align}
\dot{\bm{\xi}} &= \bm{M}(\bm{\xi}) \left( f(\bm{\xi}) - \dot{\tilde{\bm{\xi}}} \right) + \dot{\tilde{\bm{\xi}}}
\end{align}
where \(\bm{M}(\bm{\xi})\) denotes the modulation matrix and \(f(\bm{\xi})\) is the predefined DS.
The computed desired velocity $\dot{\bm{\xi}}$ is further processed by the Adaptive Motion Strategy detailed in Section~III. C. This integration ensures that the robot's trajectory is refined to accommodate the interactions of the soft regions of moving obstacles.

\subsection{Perception System Integration}



To enhance the adaptability of the proposed soft obstacle avoidance method in dynamic environments, a camera-based perception system is integrated into the framework. The overall framework consists of three interconnected components: perception system, DS modulation, and robot control. The perception system serves as an additional module that continuously monitors the workspace, employing the camera to detect, classify, and track obstacles. The estimated obstacle states, including positions, velocities, and stiffness coefficients, are transmitted to the DS modulation module. This module, based on the previously described soft obstacle avoidance method, modulates the DS to compute the modified desired velocity. The modified desired velocity is then forwarded to the robot's control, which updates the robot’s state and provides feedback to the modulation module for iterative trajectory adaptation.

\begin{figure}[h]
\centering
\includegraphics[width=0.95\linewidth]{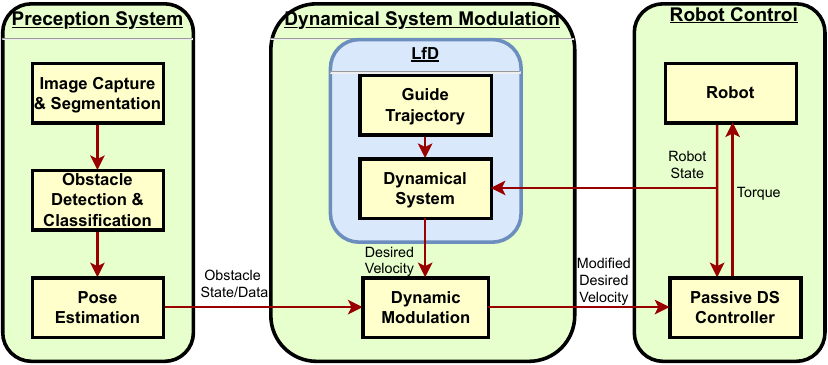}
\caption{Schematic representation of the proposed framework pipeline, illustrating the additional perception system integration alongside the DS modulation and robot control.}
\label{fig:flowchart}
\end{figure}

\section{Simulations}\label{sec:simulations}

We illustrate the results of our approach with a series of simulations that modulate the DS on trajectories from the LASA Handwriting Dataset \cite{khansari2011LearningStableNon-Linear}. To evaluate the robot's interaction with deformable obstacles, we set the attractor at the origin and randomly placed deformable obstacles within the workspace. We then assessed the robot's behavior by measuring the resulting trajectories and performance metrics. These simulations aim to validate how the learned DS, derived from demonstration data, can adapt to varying obstacle conditions, focusing particularly on deformable and non-rigid environments. In real-world scenarios, soft regions might represent deformable objects, such as cloth or foam, while hard regions could simulate rigid barriers like walls or metal. The robot's ability to modulate its response according to obstacle stiffness is critical for autonomous navigation in human environments.

First, we analyze how the robot navigates through multiple intersecting obstacles, highlighting its ability to distinguish between soft and hard regions.
Next, we evaluate the algorithm's efficiency by comparing the trajectory times when navigating through these regions. All simulations were conducted using MATLAB R2023a.

\subsection{Trajectory Analysis}
We simulated the robot's trajectory using demonstrations from the LASA Handwriting Dataset \cite{khansari2011LearningStableNon-Linear} to model the robot's interaction with the soft regions of obstacles. The dataset provides a set of reference trajectories, and we modulated the DS to guarantee smooth navigation around the deformable obstacles. These obstacles are parameterized with distinct soft and hard regions, simulating real-world deformable objects. The trajectory was computed using a linear attractor model, and the integration was performed with the fourth-order Runge-Kutta (RK4) algorithm, chosen for its balance between computational efficiency and high accuracy, making it particularly suited for handling the stiff system dynamics often encountered in environments with both soft and hard obstacles.

As shown in Fig. \ref{fig:sim_avoidance}, regardless of the robot's initial position, the simulated trajectory consistently navigates around the hard regions of the soft obstacles and successfully converges to the attractor. This demonstrates the effectiveness of our approach in avoiding obstacles while ensuring a stable trajectory to the target location.

\begin{figure}[h]
    \vspace*{\mySepBetweenTopAndFig}%
    \centering
   
    \subfloat[]{\label{fig:sshape_single}\scalebarimg{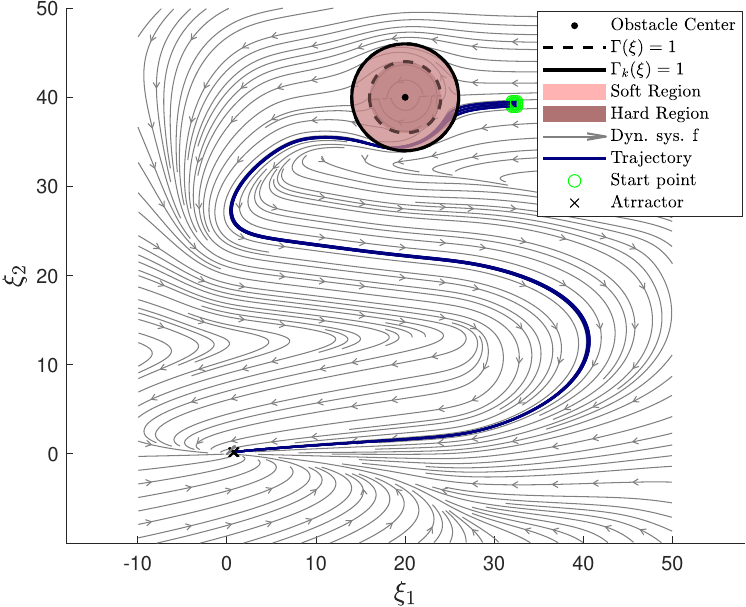}{1.1in}{}{black}}%
    \hspace{\fill}%
    \subfloat[]{\label{fig:angle_single}\scalebarimg{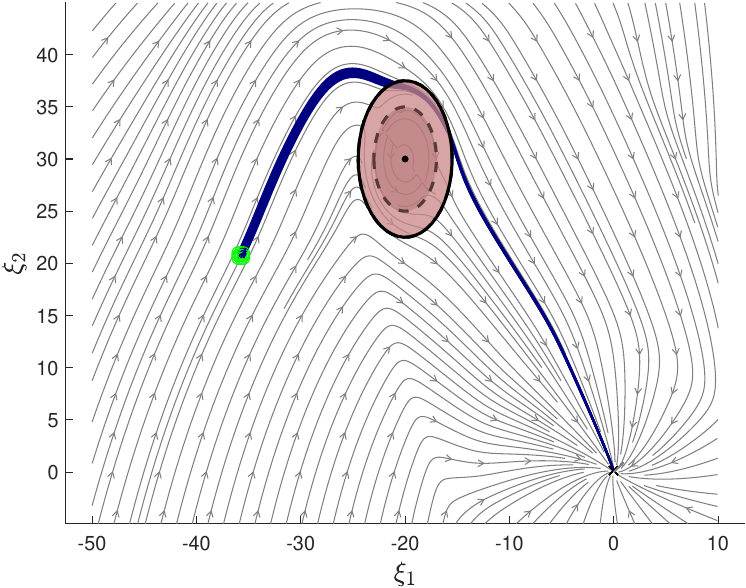}{1.1in}{}{black}}%
    \hspace{\fill}%
    \subfloat[]{\label{fig:leaf_single}\scalebarimg{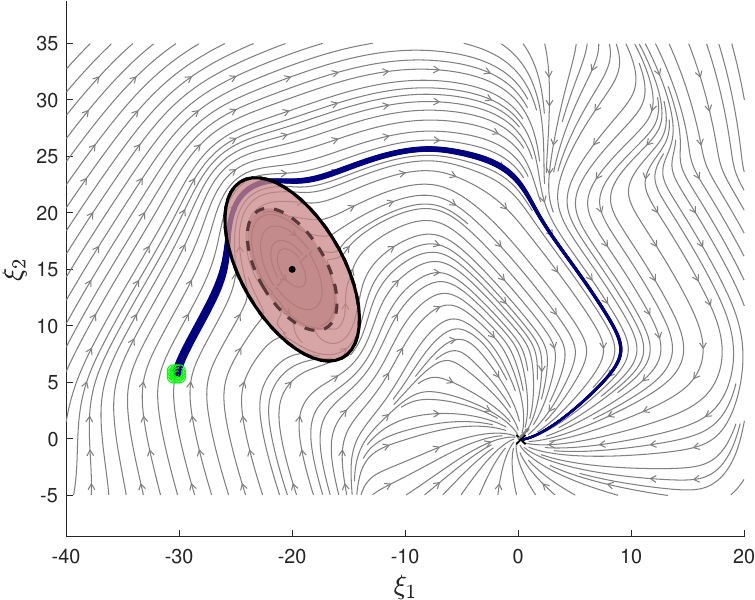}{1.1in}{}{black}}%
    \vspace*{-0.5em}

    \subfloat[]{\label{fig:sshape_inter}\scalebarimg{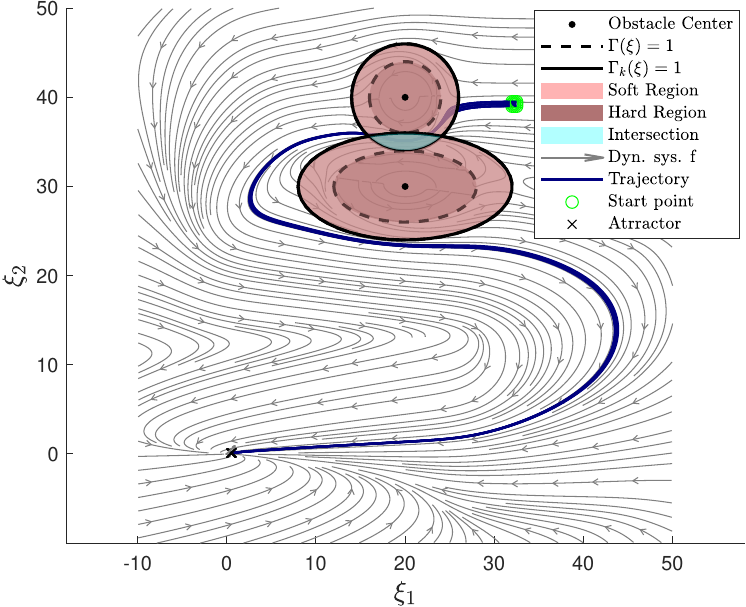}{1.1in}{}{black}}%
    \hspace{\fill}%
    \subfloat[]{\label{fig:angle_inter}\scalebarimg{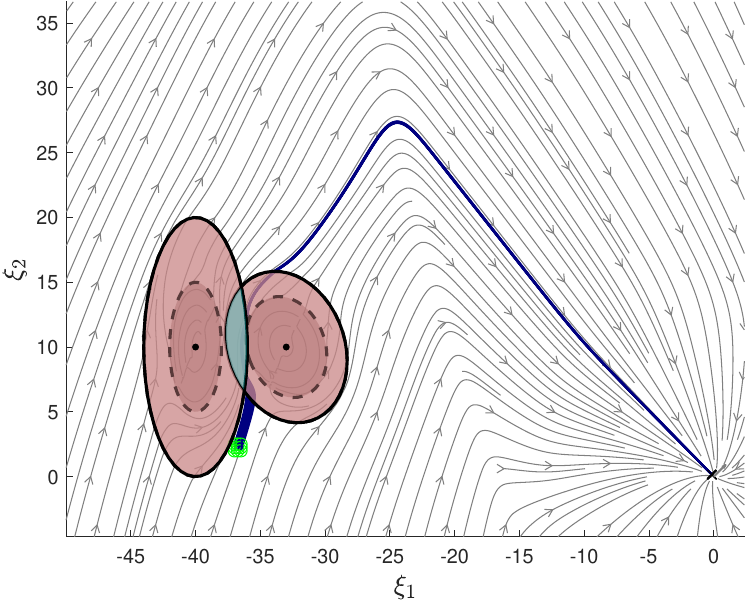}{1.1in}{}{black}}%
    \hspace{\fill}%
    \subfloat[]{\label{fig:leaf_inter}\scalebarimg{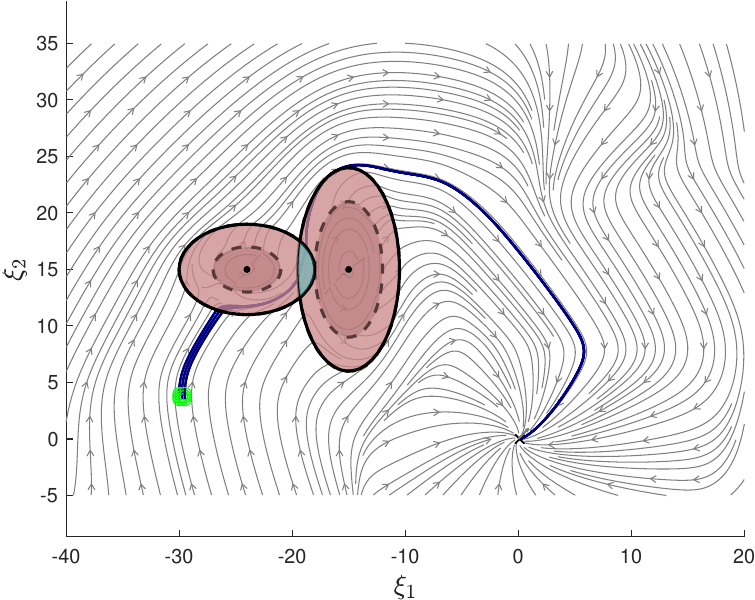}{1.1in}{}{black}}%
    \vspace*{0.5em}
    \caption{%
        Illustration of the robot navigating through the soft regions of various obstacles. The bottom row of the figures shows the robot traversing the intersection of two obstacles, highlighted by the cyan region. In these figures, the x-y axes represent the robot's 2D position. The black streamlines correspond to the modulated DS function, while the blue lines represent the robot's trajectory. The green circle indicates the robot's starting point, and the attractor is positioned at the origin.
    }%
    \label{fig:sim_avoidance}
    \vspace*{\mycaption + \mySepBetweenFigAndCap}%
\end{figure}

\subsection{Time Comparison}
 
To evaluate the robot's navigation time when passing through a deformable obstacle compared to a hard obstacle and an original algorithm \cite{huber2019avoidance} within the DS framework, we used the LASA S-shape motion trajectory. An elliptical obstacle was placed within the robot's path, and the task involved navigating from a starting point located near the obstacle to a fixed target area. This setup allows for a comparative analysis of the robot's performance in handling different types of obstacles.



\begin{figure}[h]
    \vspace*{\mySepBetweenTopAndFig}%
    \centering%
    \subfloat[]{\label{fig:compare_k}\scalebarimg{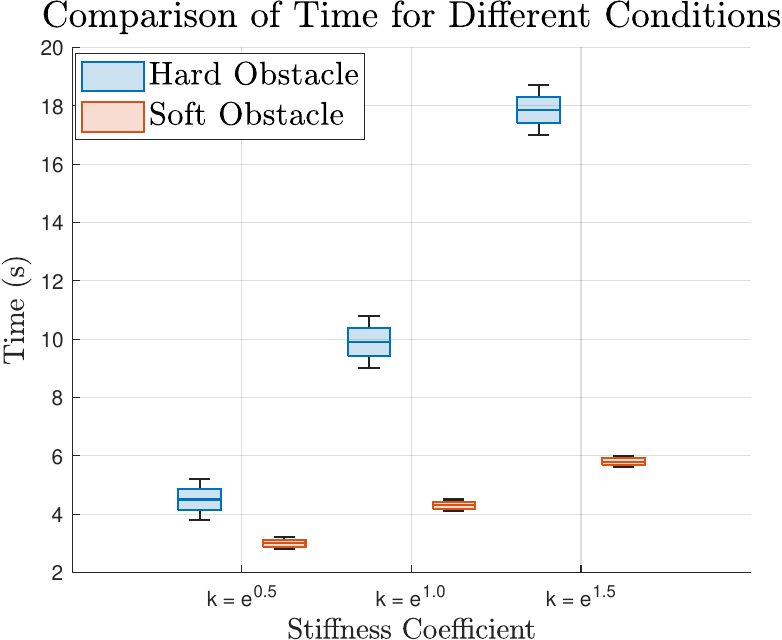}{1.627635606409449in}{}{black}}%
    \hspace{\fill}%
    \subfloat[]{\label{fig:compare_reduced}\scalebarimg{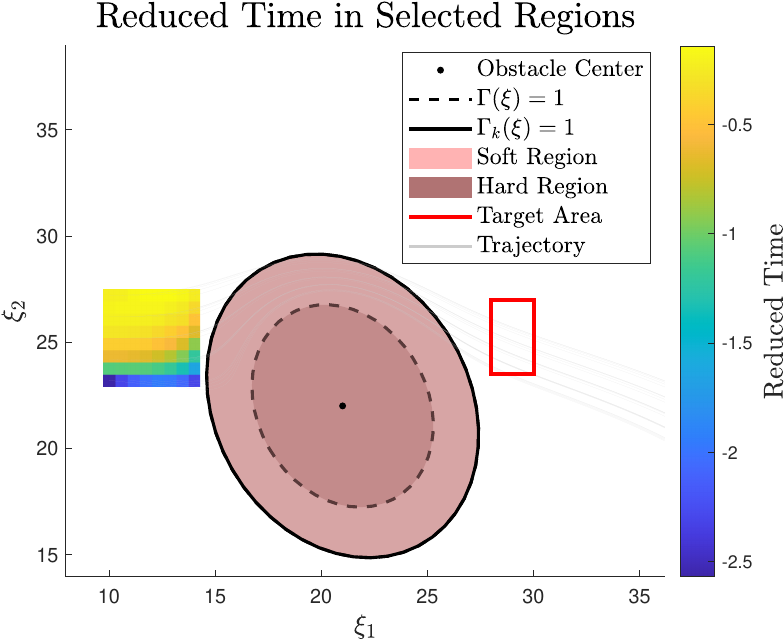}{1.627635606409449in}{}{black}}%
    \vspace{-0.3em}
    \caption{%
        (a) Comparison of different values of $k$ by starting points in selected regions [0.5x0.5] adjacent to soft obstacles. For each region, 25 points were sampled, and the time required for the trajectory to reach the target area (near the obstacle) was calculated. The required time decreases as $k$ increases. 
        (b) Comparison of time reduction in the robot path when \( k = e^{0.5} \), with 64 samples selected within a selecting area [4x4]. This value of $k$ was chosen through empirical testing, maximizing the robot's navigation efficiency. The time reduction was then evaluated against the original method, showing significant reductions, especially in paths closer to hard regions of the obstacle.
    }%
    \label{fig:comparison_figures_k}%
    \vspace*{\mySepBetweenFigAndCap}%
    \vspace{-0.5em}
\end{figure}

In Fig. \ref{fig:comparison_figures_k}, we analyze the impact of varying the parameter \( k \) on the robot’s navigation performance and time efficiency.
The parameter \( k \) in the DS model reflects the adaptation of the robot to varying environmental conditions, where higher values of \( k \) enable quicker adaptations to the obstacle's characteristics. This modulation mirrors the learned sensitivity to changes in the environment, demonstrating the strength of the LfD approach.
Fig. \ref{fig:compare_k} illustrates that as \( k \) increases, there is a notable decrease in the required navigation time, indicating that higher \( k \) values contribute to more efficient navigation, likely due to better adaptation to the soft obstacle's characteristics. Building on this, Fig. \ref{fig:compare_reduced} offers a detailed view of the time reduction achieved with \( k = e^{0.5} \). It shows a significant reduction in navigation time, particularly in regions where the path interacts with the obstacle. 

Overall, the analysis presented in these figures underscores the effectiveness of our approach in optimizing navigation time across various obstacle types. The results show that adjusting the parameter \( k \) significantly impacts the robot’s performance, with higher values leading to more efficient navigation through soft obstacles. Moreover, navigation time is substantially reduced, especially near hard regions. This comprehensive evaluation demonstrates the robustness and adaptability of our approach, validating its effectiveness in achieving both efficient and adaptive navigation strategies. The results demonstrate that the learned DS, trained via demonstration, generalizes well to novel obstacle configurations, particularly in environments with deformable objects. 

\section{Experiments}\label{sec:experiments} 

The practical applicability and real-time performance of the proposed approach are evaluated through a comprehensive set of experiments across various obstacle scenarios (outlined in Table~\ref{tab:experiment_cases}). The experiments are organized into four categories:

\begin{enumerate}
\item \textit{Benchmark Comparison} (A1-A3): Performance is compared with the original hard obstacle avoidance method \cite{huber2019avoidance} using a single deformable obstacle along trajectories from the LASA Handwriting Dataset.
\item \textit{Intersection Scenario} (B1-B4): The robot interacts with two intersecting deformable obstacles, with results benchmarked against the original method under identical DS conditions.
\item \textit{Complex Environment} (C1): A new DS is generated, and the proposed method is tested with five deformable obstacles to validate scalability.
\item \textit{Moving Obstacle} (D1): An obstacle is moved dynamically during execution to evaluate the system’s capability for real-time trajectory adaptation.
\end{enumerate}

\begin{table*}[h]
    \centering
    \vspace{-0.5em}
    \caption{Experiment cases with different obstacle configurations and avoidance methods.}
    \label{tab:experiment_cases}
    \resizebox{\linewidth}{!}{%
    \begin{tabular}{|c|c|>{\centering\arraybackslash}p{26cm}|c|}

        \hline
        \textbf{Obstacles} & \textbf{Experiment Case} & \textbf{Case Description} & \textbf{Method Used} \\
        \hline 
        \multirow{3}{*}{A - One obstacle} 
        & A1 & Hard cylinder (4 cm radius) wrapped in a soft cotton layer (6 cm total radius), \( k = e^{0.5} \) & Our method \\ \cline{2-4}
        & A2 & Hard cylinder only (6 cm radius), \( k = 1.0 \) & Original method \\ \cline{2-4}
        & A3 & Hard cylinder only (4 cm radius), \( k = 1.0 \) & Original method \\
        \hline
        \multirow{4}{*}{B - Two obstacles} 
        & B1 & Hard cylinder (4 cm radius) and elliptical cylindroid (8 cm semi-major axis, 4 cm semi-minor axis), both wrapped in a soft cotton layer, \( k = e^{0.5} \) & Our method \\ \cline{2-4}
        & B2 & Hard cylinder (6 cm radius) and elliptical cylindroid (12 cm semi-major axis, 6 cm semi-minor axis) only, \( k = 1.0 \) & Original method \\ \cline{2-4}
        & B3 & Hard cylinder (4 cm radius) and elliptical cylindroid (8 cm semi-major axis, 4 cm semi-minor axis) only, \( k = 1.0 \) & Original method \\ \cline{2-4}
        & B4 & Intersecting Hard cylinder (6 cm radius) and elliptical cylindroid (12 cm semi-major axis, 6 cm semi-minor axis), \( k = 1.0 \) & Original method \\
        \hline
       \multirow{1}{*}[-1.25em]{C - Five obstacles}
        & \raisebox{-1.25em}{C1} &   Hard cylinder (4 cm radius), \( k = 1.0 \); Intersecting Hard elliptical cylindroid (5 cm semi-major axis, 2 cm semi-minor axis), \( k = e^{1.5} \), with Hard elliptical cylindroid (8 cm semi-major axis, 4 cm semi-minor axis) wrapped in a soft cotton layer, \( k = e^{0.5} \);  and Intersecting Hard elliptical cylindroid (6 cm semi-major axis, 5 cm semi-minor axis), \( k = e^{1.0} \) with Hard elliptical cylindroid (10 cm semi-major axis, 5 cm semi-minor axis) wrapped in a soft cotton layer, \( k = e^{0.5} \)      & \raisebox{-1.25em}{Our method} \\ \cline{2-4}
        \hline
        \multirow{1}{*}{D - Moving obstacle} 
        & D1 & Hard cylinder (4 cm radius) wrapped in a soft cotton layer (6 cm total radius), \( k = e^{0.5} \) & Our method \\ \cline{2-4}
        \hline
    \end{tabular}}
    \vspace{-0.8em}
\end{table*}

\subsection{Experimental Setup}

As shown in Fig. \ref{fig:physical_obs1}, obstacles are placed on a table within the workspace of a 7-DoF Franka Emika robot\footnote{The robot is controlled using the Franka Control Interface (FCI) at 1 kHz. The control loop runs on a control PC (Intel Core i7-10700 @ 2.90GHz) installed with Ubuntu 18.04 LTS and real-time kernel (5.4.138-rt62).}\cite{haddadin2022franka}. The robot navigates this environment using Cartesian impedance control to regulate motion along the Z-axis and orientation, while passive interaction control governs the X and Y axes. A passive DS controller \cite{kronander2015passive} ensures trajectory tracking. 

The deformable obstacles were designed with soft regions filled with cotton to minimize friction forces, thereby enabling more controlled and precise interactions during the trials. Table \ref{tab:experiment_cases} presents the complete setup for the obstacles of four scenarios of experiments.
In Fig. \ref{fig:setup1.1}, we use a cylinder with a radius of 4 cm, wrapped in a cotton layer to a radius of 6 cm (see Fig. \ref{fig:setup2.1}), to achieve \( k = e^{0.5} \). In Fig. \ref{fig:setup_1.2}, we use an elliptical cylinder with a semi-major axis of 8 cm and a semi-minor axis of 4 cm, also wrapped in cotton, to achieve \( k = e^{0.5} \) (see Fig. \ref{fig:setup2.2}).

To monitor the workspace during the \textit{Moving Obstacle} experiment, an Intel RealSense Depth Camera D435i is mounted approximately 1 meter above the table, operating at 30 fps with a 640×480 pixel resolution for color and depth streams. At initialization, three AprilTags \cite{olsson2016apriltag} placed on the table establish reference points for subsequent computations. For object detection and classification, we employ YOLOv8 \cite{jocher2023yolov8}, initialized with the pre-trained YOLOv8n model and fine-tuned for 75 epochs on a custom 720-image dataset (70/30 train/validation split). Object tracking uses DeepSort \cite{wojke2017deepsort} with a pre-trained MobileNet embedder \cite{howard2017mobilenets}.
All computations run on a perception PC with an NVIDIA RTX 3090 GPU, achieving an average pipeline latency of 15 ms to support the system’s real-time requirements. Finally, the position, velocity, and stiffness coefficients of each obstacle are sent to the control PC via UDP for real-time processing.

\begin{figure}[h]
    \vspace*{\mySepBetweenTopAndFig}%
    \vspace{-4mm}
    \centering
   
    \subfloat[]{\label{fig:setup1.1}\scalebarimg{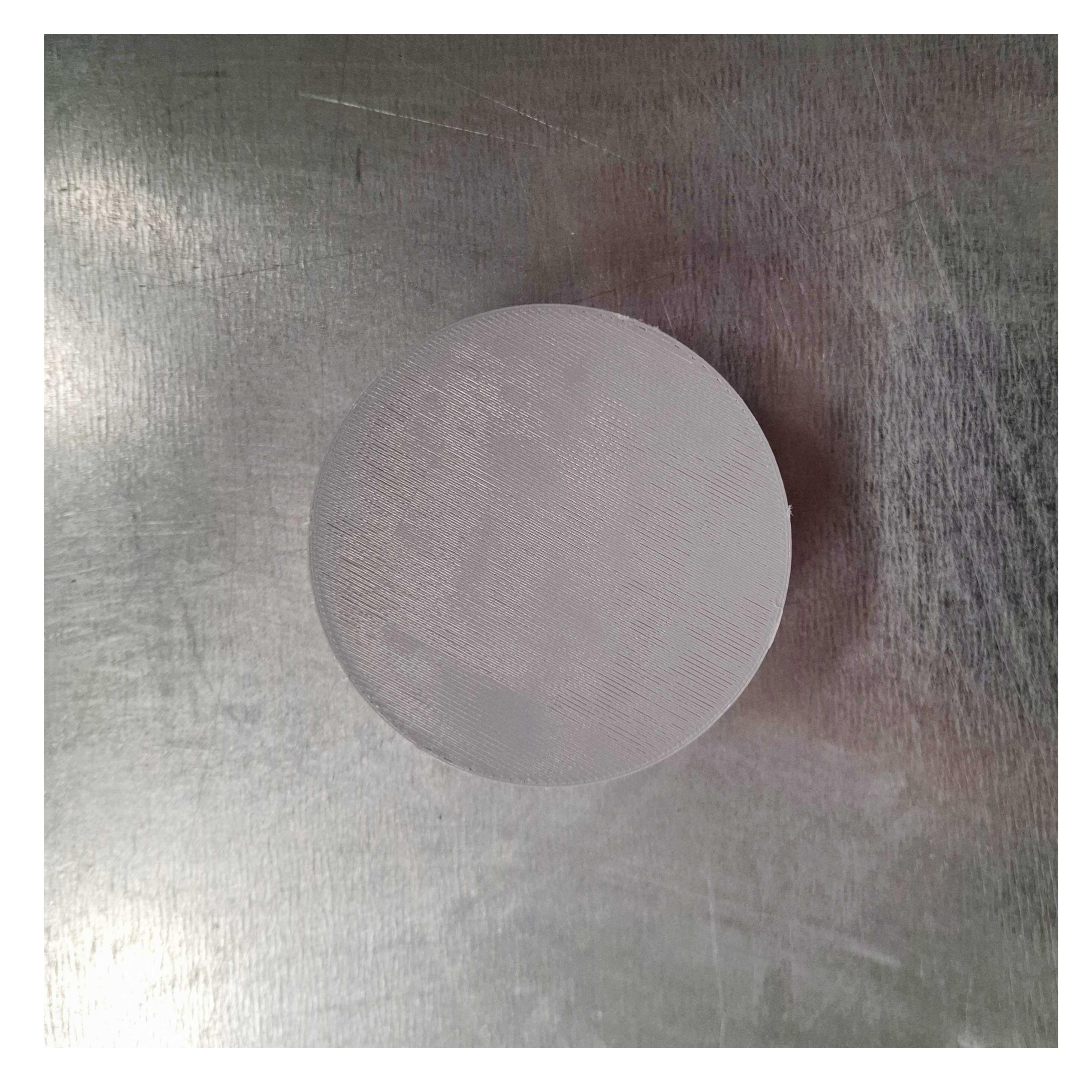}{0.23\linewidth}{}{black}}%
    \hspace{0.01\linewidth}%
    \subfloat[]{\label{fig:setup2.1}\scalebarimg{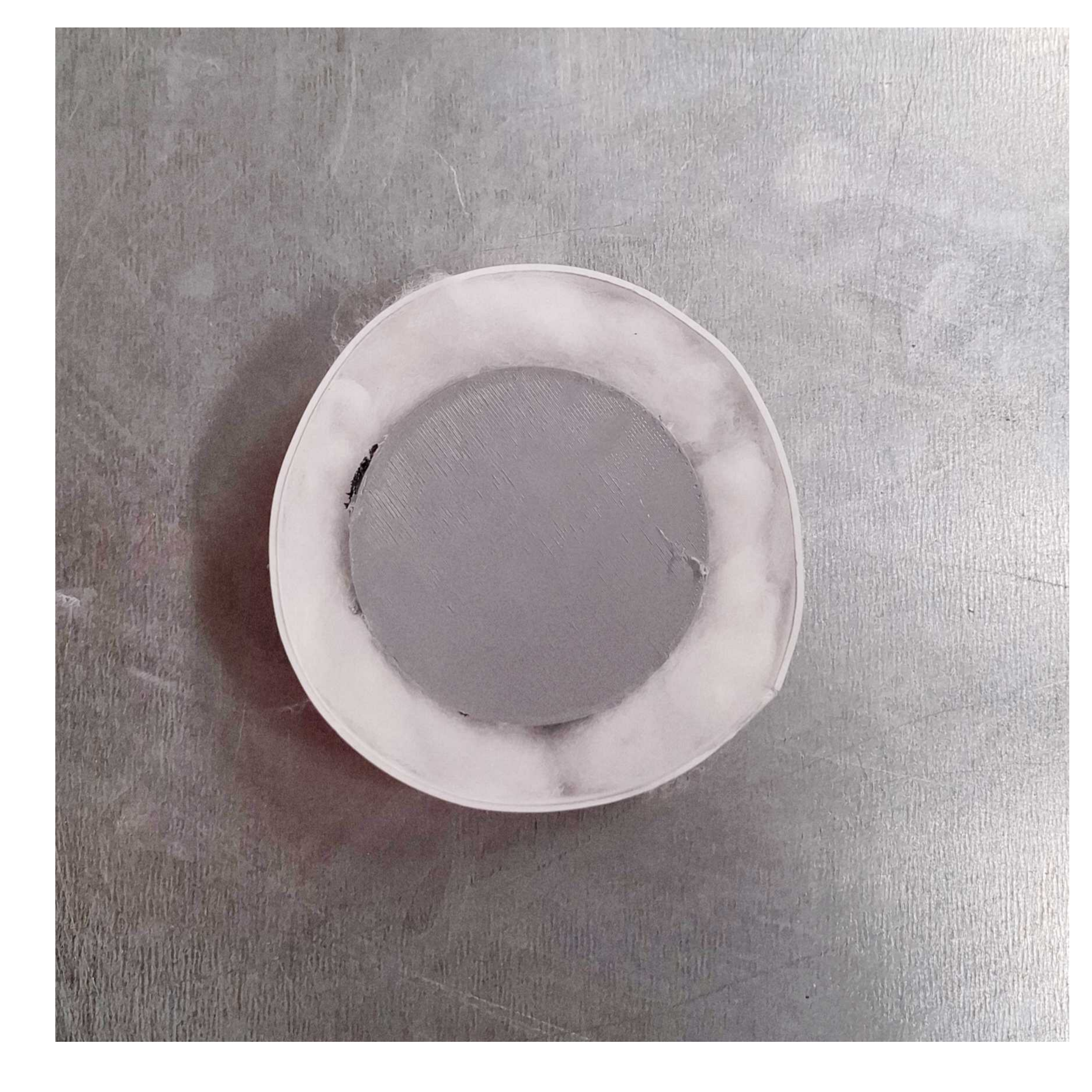}{0.23\linewidth}{}{black}}%
    \subfloat[]{\label{fig:setup_1.2}\scalebarimg{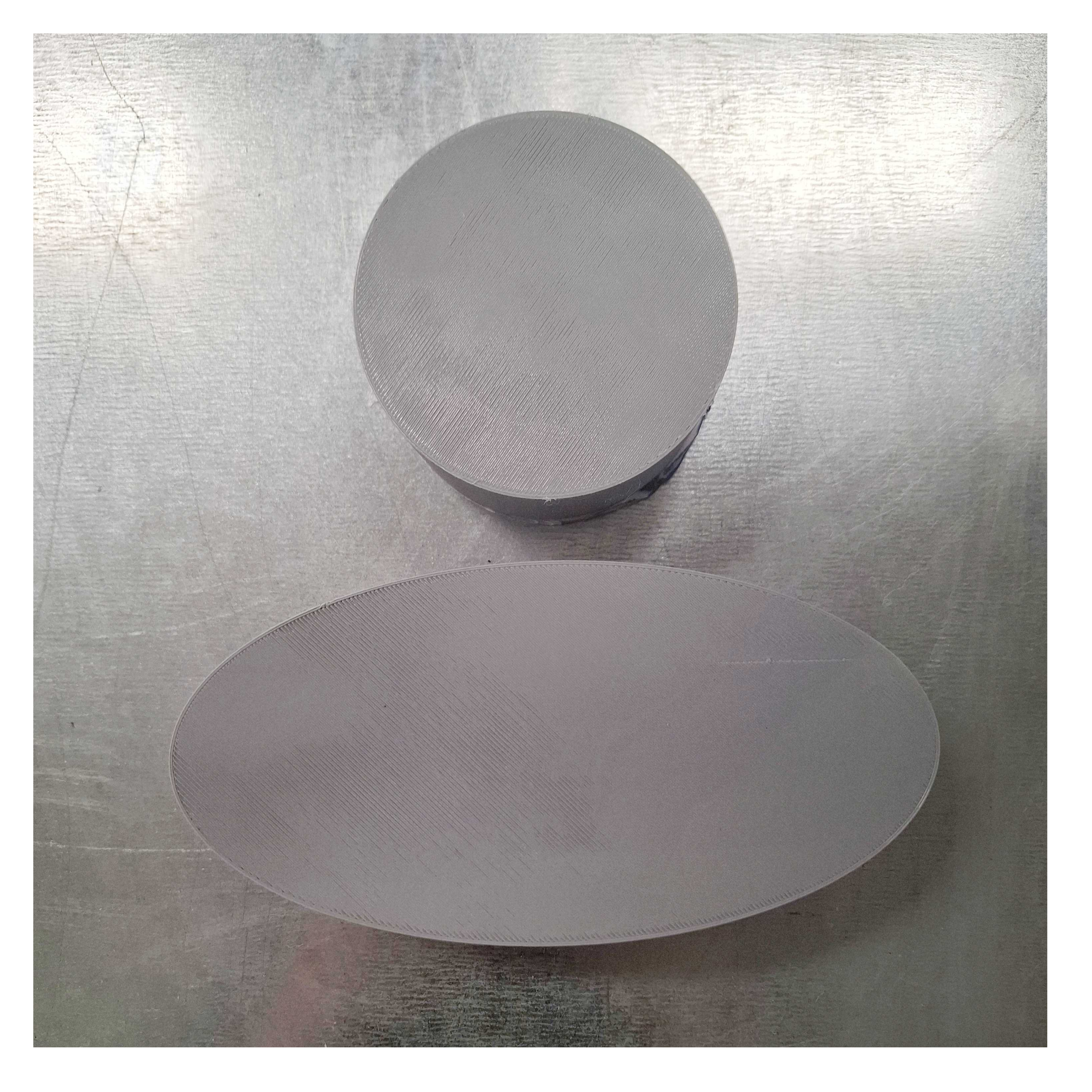}{0.23\linewidth}{}{black}}%
    \hspace{0.01\linewidth}%
    \subfloat[]{\label{fig:setup2.2}\scalebarimg{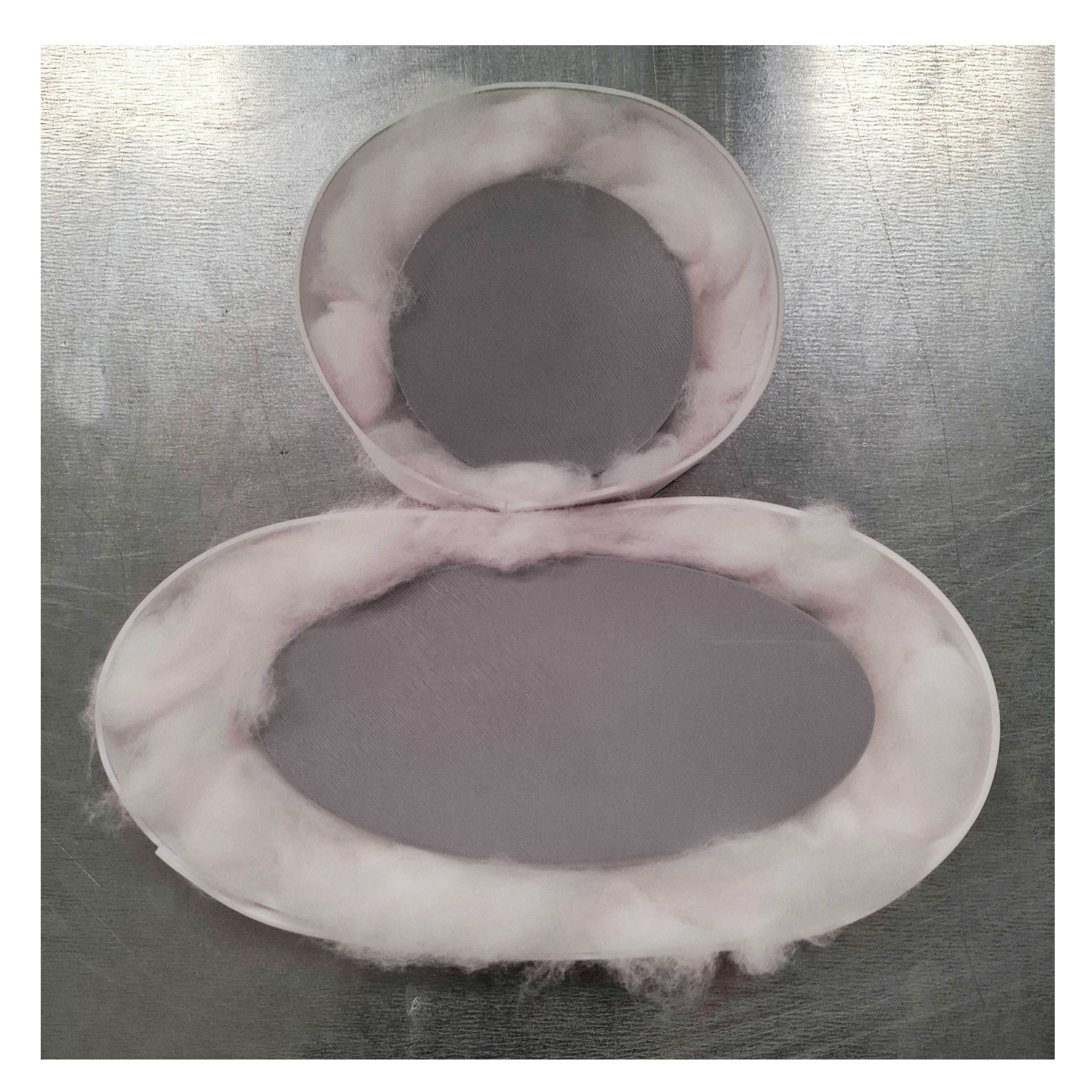}{0.23\linewidth}{}{black}}%
    
    \vspace*{-0.3em}
    \caption{(a) Hard obstacle, (b) Hard obstacle with a soft region, (c) Two hard obstacles, and (d) Two hard obstacles with intersecting soft regions.
            }%
    \label{fig:setup}
    \vspace*{\mySepBetweenFigAndCap}%
\end{figure}

\subsection{Benchmark Comparison and Intersection Scenario} 


In the \textit{Benchmark Comparison} experiments, with our method, the robot navigated through the soft region of a single obstacle, as shown in Fig. \ref{fig:exp_set1}. This resulted in an improvement in velocities compared to the original obstacle avoidance method\cite{huber2019avoidance}, which entirely avoided this region.
In the \textit{Intersection Scenario} experiments, with our method, the robot navigated through the intersection of the two deformable obstacles, reducing the trajectory velocities compared with the original method, as shown in Fig. \ref{fig:exp2_compared}. For both sets, our method showed great adherence to the expected trajectory defined by the DS, represented by the cyan line in the figures.

\begin{figure}[h]
    \vspace{-5mm}
    \centering
   
    \subfloat[]{\label{fig:expA1}\scalebarimg{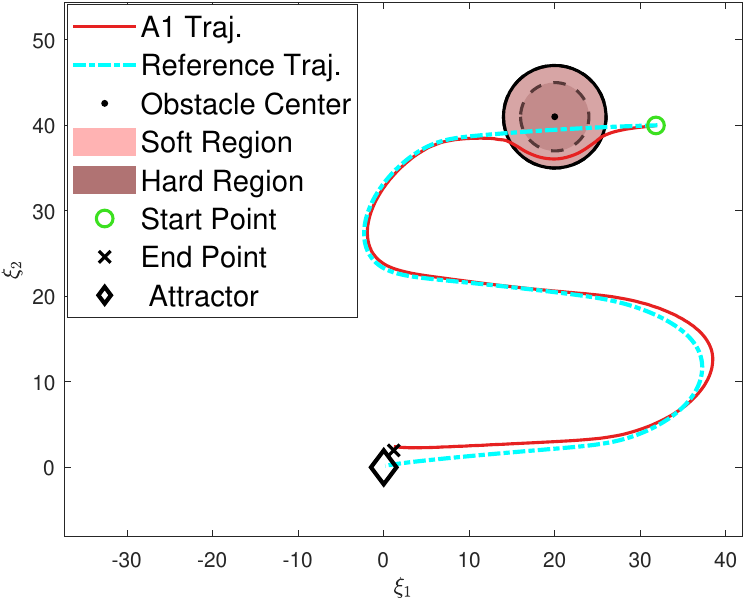}{1.1in}{}{black}}%
    \hspace{\fill}%
    \subfloat[]{\label{fig:expA2}\scalebarimg{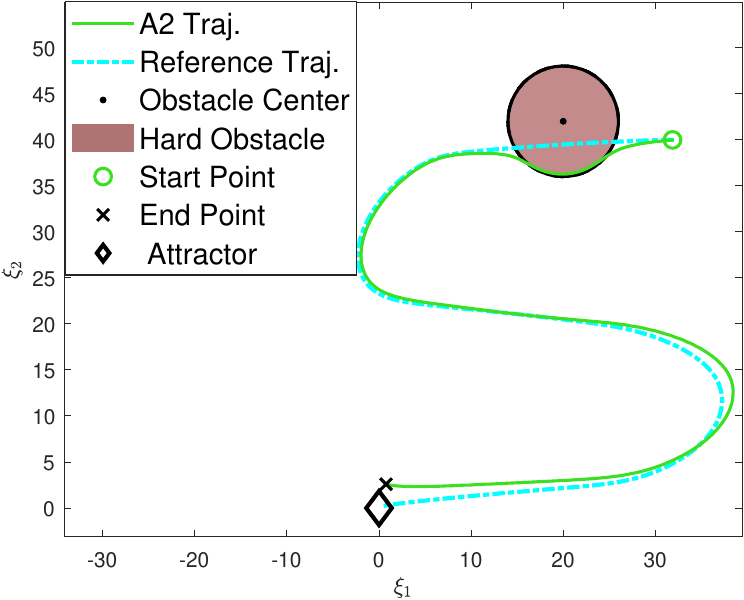}{1.1in}{}{black}}%
    \hspace{\fill}%
    \subfloat[]{\label{fig:expA3}\scalebarimg{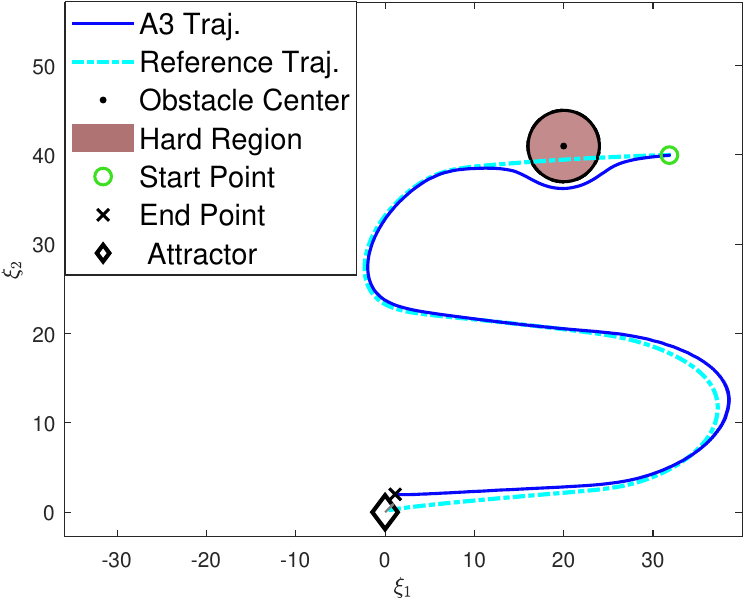}{1.1in}{}{black}}%
    \caption{Results from the three experimental configurations for the Case A are shown, in order of A1, A2 and A3 cases from left to right. The robot effectively interacted with the soft region when using our method, in contrast to the original method.}%
    \label{fig:exp_set1}
    \vspace*{\mySepBetweenFigAndCap}%
    \vspace*{-0.3em}
\end{figure}

\begin{figure}[h]
    \vspace*{-1.2em}%
    \centering
   
    \subfloat[]{\label{fig:exp2.1}\scalebarimg{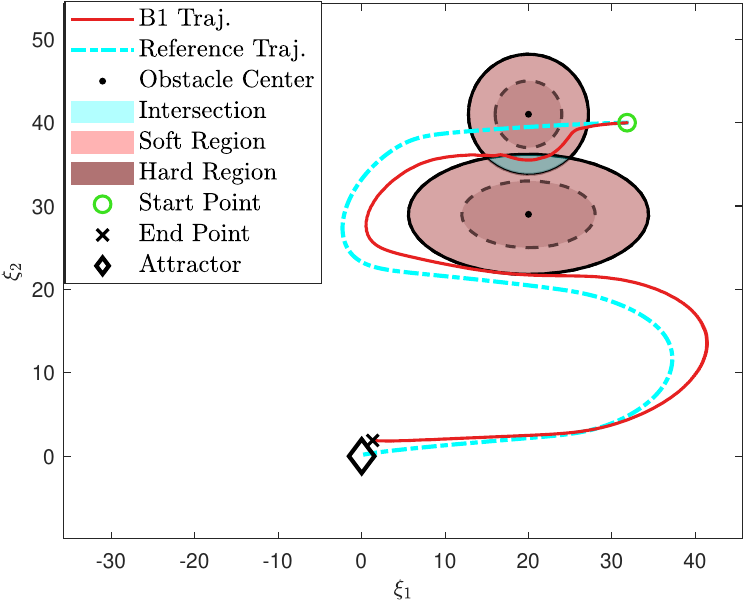}{0.48\linewidth}{}{black}}%
    \hspace{0.01\linewidth}%
    \subfloat[]{\label{fig:exp2.2}\scalebarimg{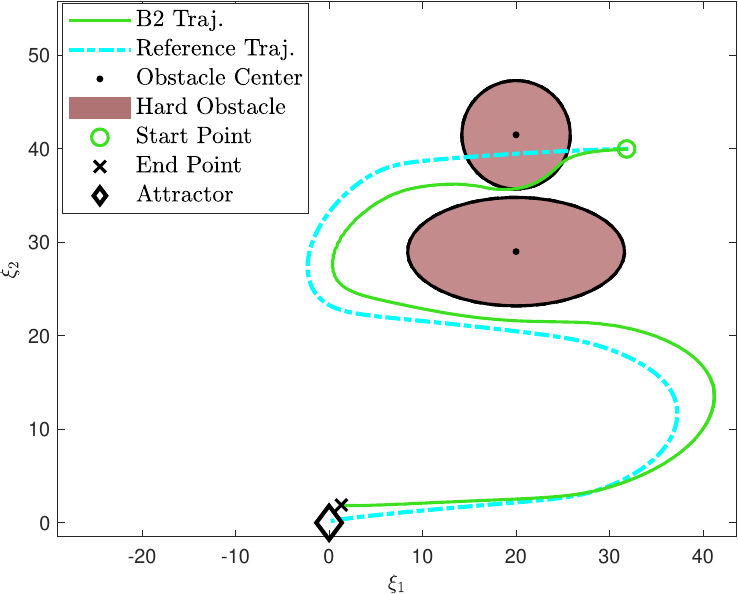}{0.48\linewidth}{}{black}}%
    
    \vspace*{-0.8em} 
    
    \subfloat[]{\label{fig:exp2.3}\scalebarimg{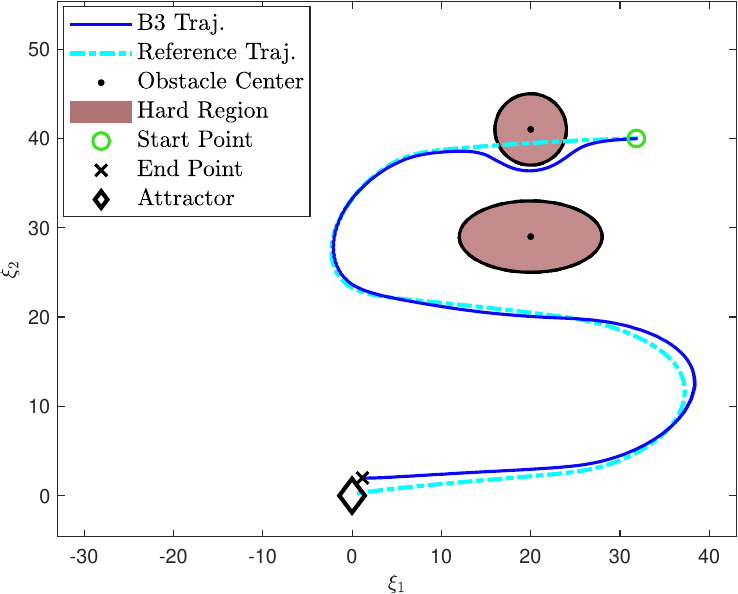}{0.48\linewidth}{}{black}}%
    \hspace{0.01\linewidth}%
    \subfloat[]{\label{fig:exp2.4}\scalebarimg{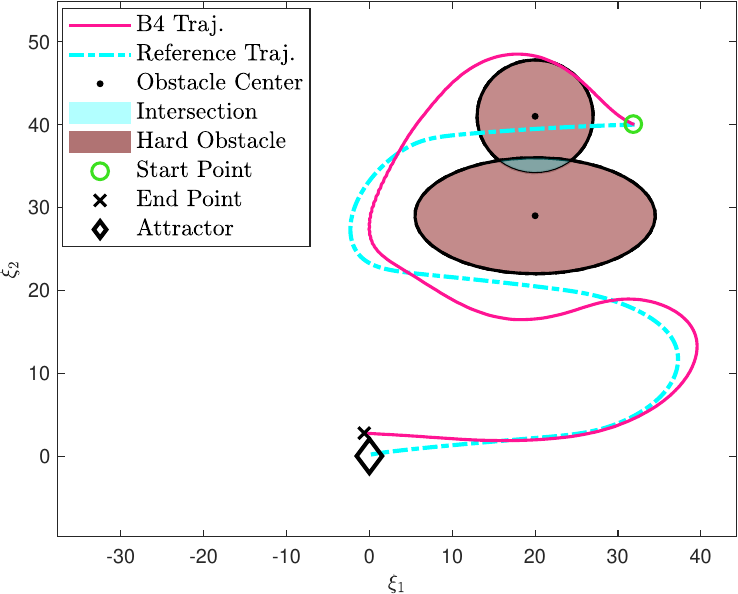}{0.48\linewidth}{}{black}}%
    
    \vspace*{-0.2em}
    \caption{Four configuration cases for the two-obstacle experiments. (a) B1: Using our method with soft obstacles. (b) B2: Using the original method, where the robot passes between two hard obstacles. (c) B3: Using the original method with two hard obstacles. (d) B4: Using the original method to avoid collision with intersecting hard obstacles.}%
    \label{fig:exp2_compared}
    \vspace*{-0.5em}
\end{figure}

Fig. \ref{fig:velocities_a} shows the robot's dynamic velocity modulation during the interaction with the soft region of a single obstacle. As proposed, the robot exhibited higer velocities in this area, with an average improvement of $26.34\%$ compared to the A3 experiment case, in which the original method was used. This performance enhancement underscores the efficacy of our method in exploiting the compliant properties of deformable obstacles to optimize navigation speed.

Furthermore, as shown in Fig. \ref{fig:velocities_b}, the velocities in the intersection area of the two obstacles decreased significantly with our method compared to the original method in the B3  experiment, with a reduction in average velocity of $32.32\%$. This substantial decrease in velocity highlights the robustness of our approach in intersecting obstacle configurations, ensuring safer and more stable robot behavior during critical interactions.

\begin{figure}[h]
    \vspace*{-0.3em}%
    \centering%
    \subfloat[]{\label{fig:velocities_a}\scalebarimg{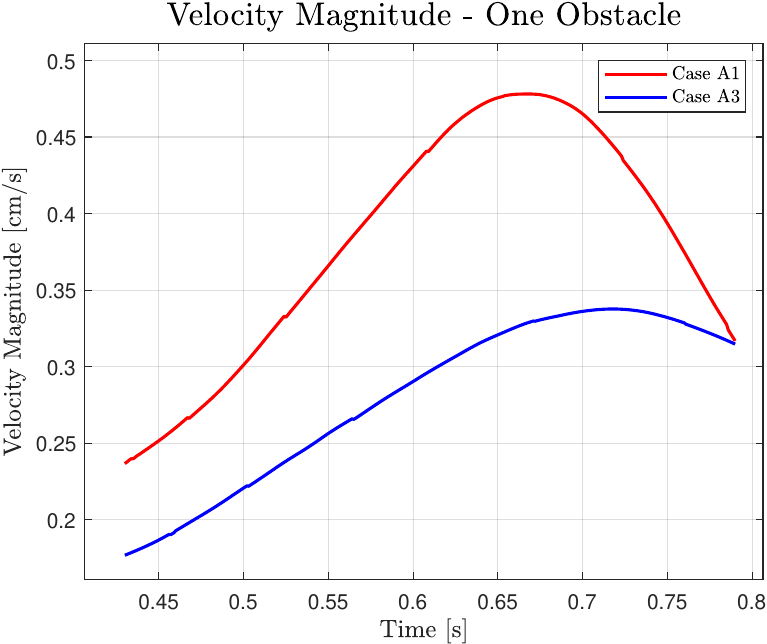}{1.627635606409449in}{}{black}}%
    \hspace{\fill}%
    \subfloat[]{\label{fig:velocities_b}\scalebarimg{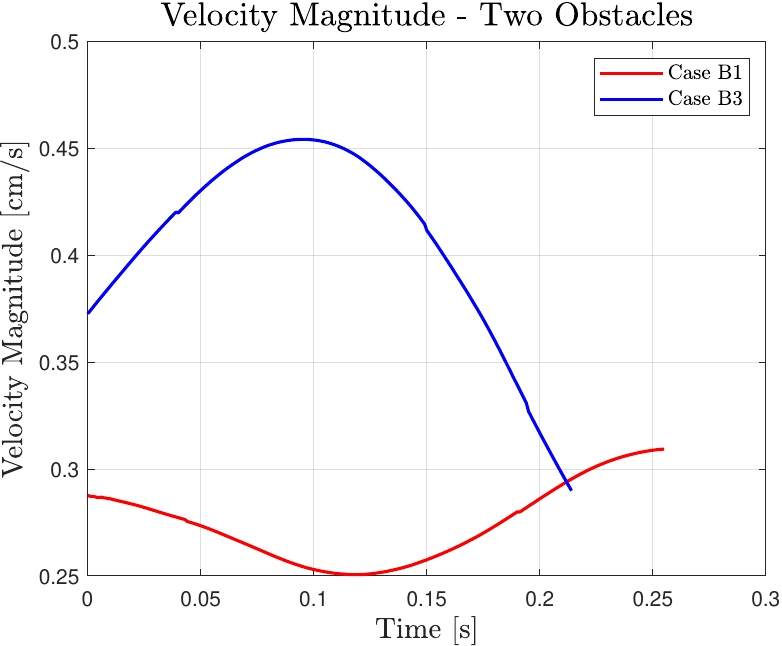}{1.627635606409449in}{}{black}}%

    \caption{%
        (a) Velocity profiles in the soft region of the obstacle, comparing our method (A1) with the original method (A3). (b) Velocity profiles in the soft region of the intersection of the two obstacles, comparing our method (B1) with the original method (B3). The starting times in the graphs correspond to the moments the robot entered the respective regions.
    }%
    \label{fig:comparison_figures}%
    \vspace{-1em}
\end{figure}

\subsection{Complex Environment}
To further validate our approach, we recorded the trajectories 10 times using the LfD approach, allowing the robot to learn the task and generate a new DS. To evaluate the adaptability of our method, we introduced five deformable obstacles into the environment: one single hard obstacle and two pairs of intersecting obstacles. The results, illustrated in Fig.\ref{fig:exp_five}, demonstrate the feasibility of our approach, where the newly generated DS enables the robot to interact effectively with multiple deformable obstacles while maintaining a smooth and adaptive trajectory.
\begin{figure}[h]
    \vspace*{\mySepBetweenTopAndFig}%
    \centering%
    \subfloat[]{\label{fig:new_DS}\scalebarimg{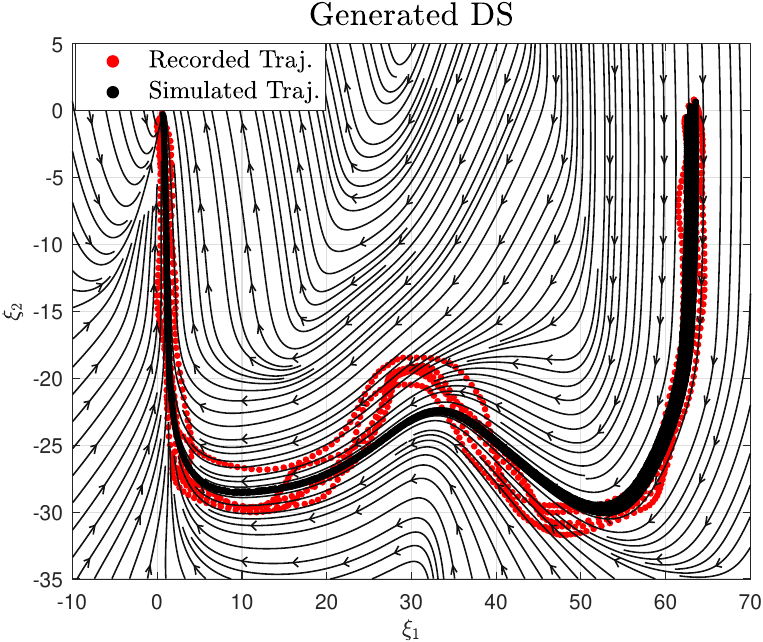}{1.627635606409449in}{}{black}}%
    \hspace{\fill}%
    \subfloat[]{\label{fig:five_traj}\scalebarimg{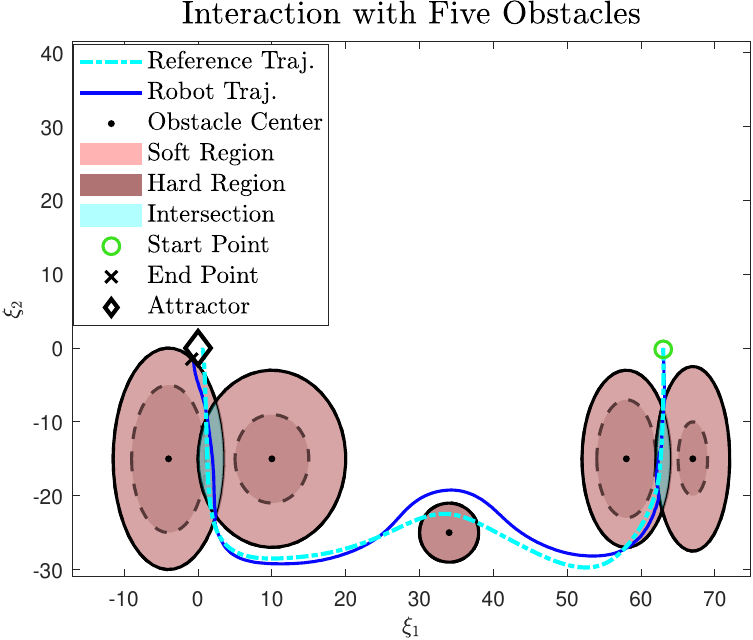}{1.627635606409449in}{}{black}}%
    \vspace{-0.3em}
    \caption{%
        (a) We recorded human demonstrations (red dots) and generated the corresponding DS, represented by black streamlines. The simulated trajectory (black dots) accurately follows this system.
        (b) In the five-obstacle experiment, the robot's trajectory (blue lines) closely mirrors the reference trajectory (cyan lines). The robot's trajectory interacts with the intersecting obstacles and effectively avoids hard obstacles. 
    }%
    \label{fig:exp_five}
\end{figure}

\subsection{Moving Obstacle}

Experiments were conducted with a single moving deformable obstacle within an S-shape DS trajectory. A human operator moves the deformable obstacle along a direction—starting from the top, then moving to the middle, and finally reaching the bottom. This path was deliberately chosen to maximize the intersections with the robot’s trajectory, thereby resulting in extensive interaction between the moving deformable obstacle and the robot. Since the obstacle was not fixed to the table plane and would be pushed away by the robot, it did not physically interact with the robot. Instead, its state was updated in DS to evaluate the real-time capability of the proposed system. 
\begin{figure}[h]
    \vspace*{\mySepBetweenTopAndFig}%
    \vspace{-3mm}
    \centering
   
    \subfloat[]{\label{fig:exp_camera1}\scalebarimg{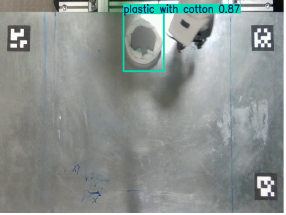}{1.1in}{}{black}}%
    \hspace{\fill}%
    \subfloat[]{\label{fig:exp_camera2}\scalebarimg{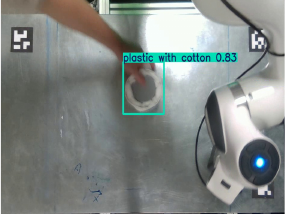}{1.1in}{}{black}}%
    \hspace{\fill} 
    \subfloat[]{\label{fig:exp_camera3}\scalebarimg{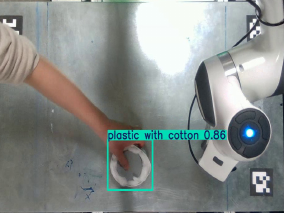}{1.1in}{}{black}}%
    \vspace*{-0.5em}
    
    \subfloat[]{\label{fig:exp_robot1}\scalebarimg{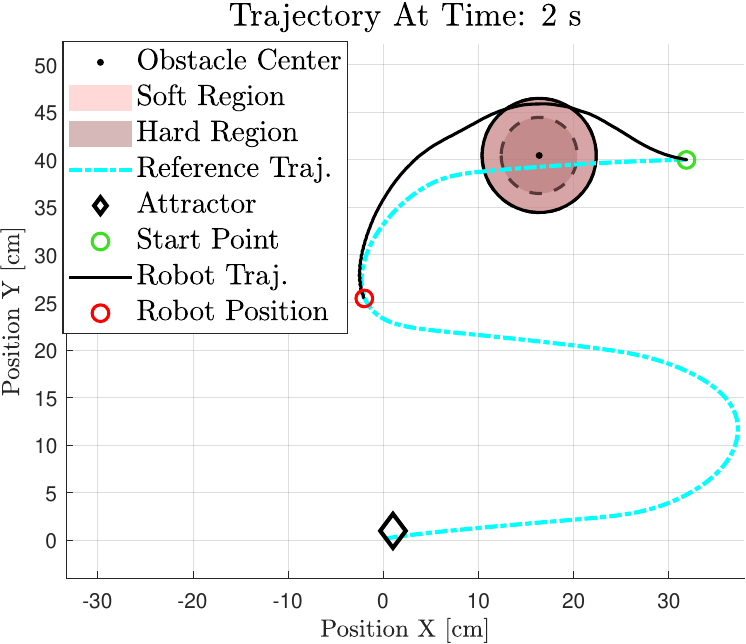}{1.1in}{}{black}}%
    \hspace{\fill}%
    \subfloat[]{\label{fig:exp_robot2}\scalebarimg{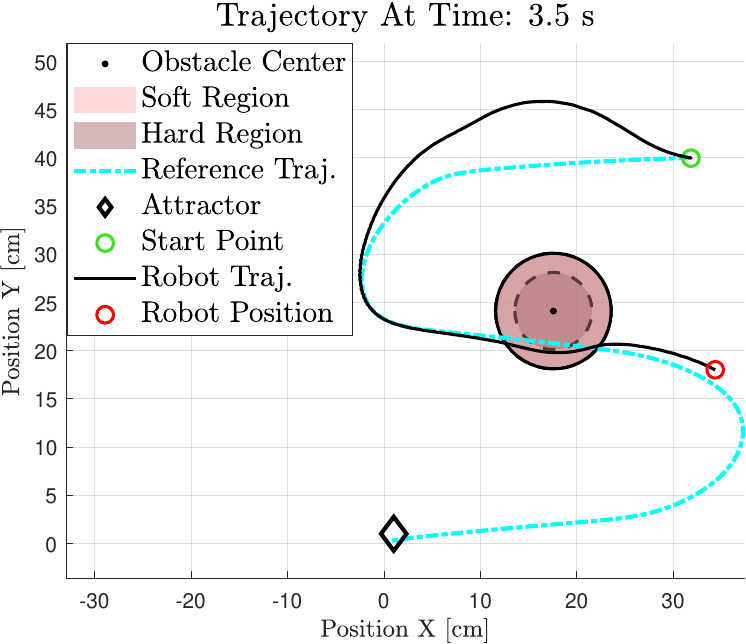}{1.1in}{}{black}}%
    \hspace{\fill} 
    \subfloat[]{\label{fig:exp_robot3}\scalebarimg{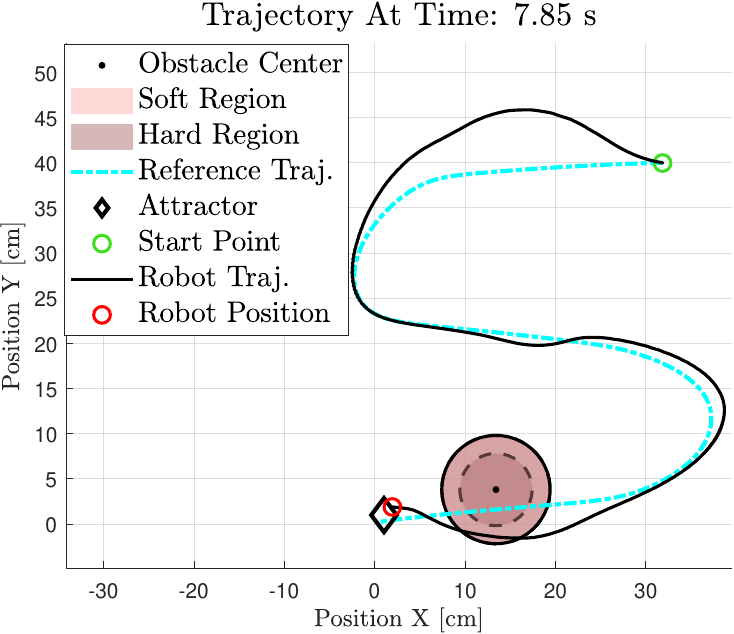}{1.1in}{}{black}}%
    
    \caption{(a)-(c) illustrate the camera captures of three distinct positions of the moving obstacle, as manipulated manually. The perception system detects the obstacle and transmits the obstacle's state.
    (d)-(f) presents the robot trajectories at three discrete time steps, incorporating the corresponding moving obstacle data. These time steps capture the dynamic process as the robot interacts with moving deformable obstacles during its trajectory. 
    }%
    \label{fig:exp_moving}
    \vspace*{\mySepBetweenFigAndCap}%
    \vspace*{0.5em}
\end{figure}

Fig. \ref{fig:exp_moving} presents snapshots from the camera, alongside a comparison among the reference trajectory, the robot's executed trajectory, and the obstacle pose. These visual captures show that the proposed system successfully detects and tracks the obstacle and computes correspondent modified robot velocity in real time using the proposed modulation method. The results demonstrate the effectiveness of our approach in robot interaction with moving deformable obstacles, ensuring smooth and responsive trajectory adaptation.

\section{Conclusion}\label{sec:conc}

In conclusion, this paper presents a novel approach for robot navigation in deformable environments by integrating LfD and DS with a dynamic modulation matrix. Our method enables real-time distinction between soft and hard regions, ensuring adaptive, stable trajectory planning. Validated through simulations and experiments, our approach effectively controls trajectory and velocity in complex environments. Future work will focus on algorithm and controller improvement accounting for both robot and deformable obstacle dynamics to ensure efficiency and safety. 





\balance

\bibliographystyle{ieeetr}
\setlength{\baselineskip}{0pt}
\bibliography{references}

\end{document}